\documentclass[10pt,journal,cspaper,compsoc,twocolumn]{IEEEtran}

\ifCLASSOPTIONcompsoc

\else

\fi

\ifCLASSINFOpdf

\else

\fi

\usepackage[ruled,vlined,linesnumbered]{algorithm2e}

\usepackage{graphicx}
\usepackage{epstopdf}
\usepackage[justification=centering]{caption}

\usepackage[cmex10]{amsmath}
\usepackage{amssymb}
\usepackage{subfigure}

\usepackage{algorithmic}

\usepackage{array}

\usepackage{mdwmath}
\usepackage{mdwtab}

\newtheorem{theorem}{Theorem}
\newtheorem{lemma}{Lemma}

\begin{document}

\title{Compound Rank-$k$ Projections \\for Bilinear Analysis}

\author{Xiaojun~Chang,
        Feiping~Nie,
        Sen~Wang,
        Yi~Yang,~\IEEEmembership{Member,~IEEE,}
        Xiaofang~Zhou,~\IEEEmembership{Senior Member,~IEEE,}
        and~Chengqi~Zhang,~\IEEEmembership{Senior Member,~IEEE}
\thanks{This work was partially supported by the National Program on Key Basic Research Project (973 Program) under Grant 2015CB352302, partially supported by the ARC DECRA project under Grant DE130101311 and partially supported by the National Natural Science Foundation of China under Grant 61303143.}
\thanks{X. Chang, Y. Yang and C. Zhang are with Center for Quantum Computation and Intelligent Systems, University of Technology Sydney, Australia. Email: \{uqxchan1@cs.cmu.edu,yi.yang@uts.edu.au, chengqi.zhang@uts.edu.au\}}
\thanks{F. Nie is with Center for OPTical Imagery Analysis and Learning, Northwestern Polytechnical University, Shaanxi, China. Email: feipingnie@gmail.com}
\thanks{S. Wang and X. Zhou are with the school of Information Technology and Electrical Engineering, The University of Queensland, Queensland 4072. Email: \{sen.wang\}@itee.uq.edu.au}
}

\markboth{Journal of \LaTeX\ Class Files,~Vol.~X, No.~X, XXXXXXX~20XX}%
{Chang \MakeLowercase{\textit{et al.}}: Bare Demo of IEEEtran.cls for Computer Society Journals}

\IEEEcompsoctitleabstractindextext{%
\begin{abstract}
In many real-world applications, data are represented by matrices or high-order  tensors. Despite the promising performance, the existing two-dimensional discriminant analysis algorithms employ a single projection model to exploit the discriminant information for projection, making the model less flexible. In this paper, we propose a novel Compound Rank-$k$ Projection (CRP) algorithm for bilinear analysis. CRP deals with matrices directly without transforming them into vectors, and it, therefore, preserves the correlations within the matrix and decreases the computation complexity. Different from the existing two-dimensional discriminant analysis algorithms, objective function values of CRP increase monotonically. In addition, CRP utilizes multiple rank-$k$ projection models to enable a larger search space in which the optimal solution can be found. In this way, the discriminant ability is enhanced. We have tested our approach on five datasets, including UUIm, CVL, Pointing'04, USPS and Coil20. Experimental results show that the performance of our proposed CRP performs better than other algorithms in terms of classification accuracy.
\end{abstract}

}

\maketitle

\IEEEdisplaynotcompsoctitleabstractindextext

\IEEEpeerreviewmaketitle

\section{Introduction}
\IEEEPARstart{L}{inear} discriminant analysis (LDA), also known as Fisher's Linear Discriminant (FLD), is a classical method for data representation and feature extraction. FLD is commonly utilized in the fields of computer vision and pattern recognition \cite{bib_Jieping,bib_Zou,bib_Yang_Wen,LRLDA}. For example, Peter N. Belhumeur et al. \cite{bib_Belhumeur} use LDA to represent facial expression images efficiently. The classical LDA aims to find a set of vectors so as to maximize the trace of between-class scatter matrix while minimizing the trace of within-class scatter matrix in the transformed feature space. Recent works have indicated that it is more natural and beneficial to represent an image with a matrix since exploiting the neighborhood information of a certain pixel is essential to the performance. The classical LDA, however, requires that an image is represented by a vector. This vectorization method has some inherent drawbacks. Firstly, it will erase the correlations within the matrix. Secondly, the data dimensionality increases when we transform the matrix representation into the vector representation \cite{bib_Ma,HouNYW13}. Hence, the computational burden is increased dramatically. 

To address these problems, two-dimensional linear discriminant analysis and its variants have been studied over the last decade. Instead of vectorizing the matrices before dimension reduction, the two-dimensional linear discriminant analysis works with data in a matrix representation, which can preserve the spatial correlation of the original data and reduce the computation complexity. Liu et al. \cite{Algebraicfs} propose to use an optimal discriminant criterion to extract algebraic features, calculating a set of optimal discriminant projection vectors according to a generalized Fisher criterion function. The classical 2DLDA, proposed by Ye et al. \cite{bib_Ye}, aims to learn a single set of projection matrices and introduces an iterative algorithm. In their experimental results, performance is stable when the number of iteration increases. Thus, only one iteration step is required in their experiments.

Inoue et al. \cite{bib_Inoue}, however, have pointed out the iterative algorithm proposed in \cite{bib_Ye} does not necessarily guarantee the monotonicity of the objective function value, which is mainly caused by the singularity of the between-class scatter matrix. In \cite{bib_Inoue}, they present a simple method, namely Selective Algorithm for 2DLDA, to select transformation matrices with a higher discriminant ability. In addition to the proposed algorithm, they also propose a non-iteratively parallel algorithm, which transforms rows and columns of the matrices independently. Although this method seems promising from the theoretical perspective, their experimental results show that its recognition rate is a little worse than the classical 2DLDA.

Another limitation of all the existing discriminant analysis algorithms  \cite{bib_Jieping,bib_Zou,bib_Yang_Wen,bib_Belhumeur,bib_fan,bib_Franck,bib_Ye,bib_Xiong,bib_YangYi} is that their performance suffer from the balance between the degree of freedom and the avoidance of the over-fitting problem. Specifically, LDA produces multiple full rank bilinear projections, which has the largest degree of freedom but induces the well-known over-fitting problem. To overcome this problem, the authors \cite{rank1} has proposed multiple rank-1 projection method based on the principal component. Its experimental results demonstrate that this method is at its best when the number of training samples is very small. The classical 2DLDA produces multiple rank-1 projections and has better performance in dealing with the over-fitting problem, but has a much smaller degree of freedom. Motivated by these observations, we intend to increase the degree of freedom while avoiding the over-fitting problem.  

In this paper, we aim to solve the above limitations of the existing discriminant analysis algorithms for high-order data and propose a compound rank-$k$ projection algorithm for discriminant bilinear analysis. Different from \cite{bib_Ye}, the convergence of our optimization approach is explicitly guaranteed. We adopt multiple orthogonal projection models to obtain more discriminant projection directions. In particular, we use $h$ sets of projection matrices to find a low-dimensional representation of the original data. The $h$ projection matrices are orthogonal to each other. By doing so, we can project the original data into different orthogonal basis and information from various perspectives can be obtained. The key novelty of our method is that it adopts multiple projection models, which are integrated and work collaboratively. In this way, a larger search space is provided to find the optimal solution, which will yield better classification performance.

We name the proposed algorithm as Compound Rank-k Projection for Bilinear Analysis (CRP). It is worthwhile noting that the algorithm can be readily extended to high-order tensor discriminant analysis. The main contributions of our work can be summarized as follows:
\begin{enumerate}
\item CRP can deal with matrix representations directly without converting them into vectors. Hence, spatial correlations within the original data can be preserved. Compared with the conventional algorithms, the computation complexity is reduced.  

\item Compared to the classical 2-dimensional linear discriminant analysis methods \cite{bib_Ye,bib_Xiong,bib_Li}, CRP benefits from the trade-off between the degree of freedom and the avoidance of the over-fitting problem.

\item Although the classical 2DLDA gains good performance, its iterative optimization algorithm may not converge due to the singularity of the between-class scatter matrix. Differently, the convergence of our algorithm is explicitly guaranteed. 
\end{enumerate}

The rest of this paper is organized as follows: Section 2 summarizes an overview of the classical LDA as well as 2DLDA. A novel compound rank-$k$ projection for bilinear analysis is proposed in section 3. We present our experimental results on five different datasets in section 4. The conclusion of our work is discussed in section 5.

\section{Related work}
\subsection{Classical LDA}
The conventional LDA aims to project the original high-dimensional data to a lower dimensional subspace for better classification performance \cite{bib_Deng,bib_Bingbing,YanLRS13,YanRLSS14}. The original dataset is denoted as $X \in {\mathbb{R}}^{l \times n}$, which is grouped into $c$ classes $\pi = \{ \pi_1, \pi_2,..., \pi_c \}$. $\pi_i$ contains $n_i$ data points from the $i$-th class. The transformation of the classical LDA to a lower dimensional subspace is $y_i = W^Tx_i$, where $W \in \mathbb{R}^{l \times c}$ and $x_i$ is a vector representation of the original data. By finding the best transformation matrix $W$, data points from different classes become more separated while data points from the same class become more compact after the transformation \cite{bib_Kim,bib_Bandos,WuWYZN10,HanWJZY10,NieXSZ09}. In this way, better classification performance is achieved.

More specifically, two scatter matrices in LDA, namely between-class matrix $S_b$ and within-class matrix $S_w$, are defined as follows:
\begin{equation}\nonumber
S_b=\sum_{i=1}^c n_i(M_i-M)(M_i-M)^T
\end{equation} and 
\begin{equation}\nonumber
S_w=\sum_{i=1}^c \sum_{X_j \in \pi_i} (X_j-M_i)(X_j-M_i)^T,
\end{equation} where $n_i$ is the number of data samples in the $i$-th class and $X_j$ is the $j$-th sample in the $i$-th class. $M_i=\frac{1}{n_i} \sum_{X_j \in \pi_i} X_j$ is the mean of the $i$-th class, and $M=\frac{1}{n}\sum_{i=1}^c \sum_{X_j \in \pi_i} X_j$ is the global mean.

In a lower dimensional subspace, the between-class scatter matrix and the within-class scatter matrix are transformed to $\tilde{S_b}=W^TS_bW$ and $\tilde{S_w}=W^TS_wW$, respectively, according to \cite{bib_Hou}. The objective function is defined as follows:
\begin{small}
\begin{equation}\label{opt_LDA}
\max_{W} Tr((W^TS_wW)^{-1}(W^TS_bW)),
\end{equation}
\end{small}where $Tr(\cdot)$ denotes the matrix trace operation. The objective function aims to find the best $W$ to maximize the trace of the transformed between-class scatter matrix $\widetilde{S_b}$  and minimize the trace of the transformed within-class scatter matrices $\tilde{S_w}$. This objective function can be solved by eigen-decomposition of $(S_w)^{-1}S_b$ in \cite{bib_Stuhlsatz,bib_Zafeiriou,bib_Jia}. However, as we mentioned in section 1, the classical LDA has some inherent drawbacks. To tackle these problems, two-dimensional LDA has been proposed.

\subsection{Classical 2DLDA}

Different from the classical LDA, 2DLDA uses the matrix representation instead of the vector representation. We denote a set of data as $\textbf{X}=\{X_1, X_2, ..., X_n\}, X_i\in \mathbb{R}^{l_1 \times l_2}$, which are grouped into $c$ different classes $\pi_1, ..., \pi_c$. The goal of 2DLDA is to seek a single set of transformation matrices, $U$ and $V$, projecting the original data into a lower dimensional subspace \cite{bib_Yang}. In this subspace, the two newly transformed matrices can be computed as follows:
\begin{equation}\nonumber
\widetilde{S_w}=\sum_{i=1}^c \sum_{X_j \in \pi_i}U^T(X_j-M_i)VV^T(X_j-M_i)^TU
\end{equation}

\begin{equation}\nonumber
\widetilde{S_b}=\sum_{i=1}^c n_iU^T(M_i-M)VV^T(M_i-M)^TU
\end{equation}

The objective function is defined as follows:
\begin{equation}
\max_{U, V} f = \max_{U, V} Tr((\widetilde{S_w})^{-1} \widetilde{S_b})
\label{of2d}
\end{equation}

As mentioned before, the key idea of the classical 2DLDA is to find the optimal $U$ and $V$ which maximize the objective function value $f$ in \eqref{of2d}. Since it is difficult to compute the optimal $U$ and $V$ simultaneously, Ye et al. \cite{bib_Ye} propose an iterative algorithm. However, Inoue et al. \cite{bib_Inoue} have pointed out that this iterative algorithm cannot guarantee the monotonicity of the objective function value $f$ and it is hard to determine appropriate termination criteria. To promise the monotonicity, the authors adopt trace ratio and trace difference in \cite{optimaldim}.

\section{Compound Rank-$k$ Projection}
In this section, we describe in detail our proposed algorithm. We define a set of data points as $\textbf{X} = \{X_1, X_2, ..., X_n\}, X_i\in \mathbb{R}^{l_1 \times l_2}$, which is grouped into $c$ different classes $\{\pi_1, ..., \pi_c\}$. In contrast to the classical 2DLDA, our proposed approach aims to seek $h$ optimal sets of $\mathcal{U}=\{U_1, U_2, ..., U_h\}, U_i \in \mathbb{R}^{l_1 \times k}$ and $\mathcal{V}=\{V_1, V_2, ..., V_h\}, V_i \in \mathbb{R}^{l_2 \times k}$ to project the original data points into $h$-dimensional subspaces: 
\begin{equation}\nonumber
\widetilde{X_i} = \{Tr(U_1^TX_iV_1), Tr(U_2^TX_iV_2), ..., Tr(U_h^TX_iV_h)\}
\end{equation}
The primary goal of our approach is that it employs multiple models to provide a larger space to find the optimal solution. In this way, the degree of freedom is increased. To put it from another way, we enhance the discriminant ability.

It is worth noticing that there is a trade-off between the degree of freedom and the avoidance of the over-fitting problem. To be more specific, when we allow the rank of bilinear projection $UV^T$ to be full-rank, the proposed algorithm can be reduced to the classical LDA, which has the largest degree of freedom but induces the well-known over-fitting problem. The classical 2DLDA has better performance in dealing with the over-fitting problem but has a much smaller degree of freedom. Compared with these algorithms, our proposed algorithm can benefit from this trade-off. 

Suppose we have extracted the first $p-1$ dimensional features, now we begin to extract the $p$-th dimension and make it orthogonal to the first $p-1$ dimensions. First we perform an orthogonal transformation on the data:

\begin{small}
\begin{equation}\nonumber
vec(X) \leftarrow vec(X)-vec(U_{p-1}V_{p-1}^T)(vec(U_{p-1}V_{p-1}^T))^Tvec(X)
\end{equation}
\end{small}

The key point of this orthogonal transformation is that we can project the original data into different orthogonal basis so as to get various information from differing perspectives.

According to Lemma \ref{lemma1} in Appendix, we obtain:
\begin{small}
\begin{equation}
vec(X) \leftarrow vec(X)-Tr(U_{p-1}^TXV_{p-1})vec(U_{p-1}V_{p-1}^T)
\label{eq_x}
\end{equation}
\end{small}
Then we compute the optimal solution of the following objective function for the transformed data.

\begin{small}
\begin{equation}\nonumber
\max \sum_{p=1}^h \frac{ \sum\limits_{i=1}^c Tr(U_p^T(\overline{X_i} - \overline{X})V_p)^2}{\sum\limits_{i=1}^c \sum\limits_{X_j \in \pi_i} Tr(U_p^T(X_j- \overline{X_i})V_p)^2},
\end{equation}
\end{small}
\begin{equation}\nonumber
\text{s.t}.~~vec(U_pV_p^T)^Tvec(U_pV_p^T)=1
\end{equation}

In order to avoid over-fitting and singularity of the within-class scatter matrix, a regularization term is added to the objective function. We can rewrite the objective function as follows.

\begin{small}
\begin{equation}
\max \sum_{p=1}^h \frac{\sum\limits_{i=1}^cTr(U_p^T(\overline{X_i} - \overline{X})V_p)^2}{\sum\limits_{i=1}^c \sum\limits_{X_j \in \pi_i} Tr(U_p^T(X_j- \overline{X_i})V_p)^2 + \lambda Tr(U_pV_p^TV_pU_p^T)},
\end{equation}
\end{small}
\begin{equation}\nonumber
\text{s.t.}~~vec(U_pV_p^T)^Tvec(U_pV_p^T)=1
\end{equation}

According to Lemma \ref{lemma3} in the appendix, the equation can be transformed to
\begin{footnotesize}
\begin{equation}
\max \sum_{p=1}^h \frac{\sum\limits_{i=1}^c Tr(U_p^T(\overline{X_i}-\overline{X})V_p)^2}{\sum\limits_{i=1}^c \sum\limits_{X_j \in \pi_i} Tr(U_p^T(X_j-\overline{X_i})V_p)^2 + \lambda Tr(U_pV_p^TV_pU_p^T)},\label{eq_opt}
\end{equation}
\end{footnotesize}
\begin{equation}\nonumber
\text{s.t.}~~Tr(U_p^TU_pV_p^TV_p)=1
\end{equation}

The optimal $U_p$ and $V_p$ would maximize the objective function. Since it is difficult to compute the optimal $U_p$ and $V_p$ simultaneously, we present an iterative algorithm. To be more specific, for a fixed $V_p$, we can obtain the optimal $U_p$ by solving the optimization problem that is quite similar to Eq. \eqref{of2d}. Afterwards, $V_p$ is similarly updated by using the obtained $U_p$. Note that our algorithm can promise a monotonic increase of the objective function. 

\section{optimization}
In this section, we propose an iterative approach to optimizing the objective function in \eqref{eq_opt}. Specifically, for a fixed $V_p$, the objective function equals to:

\begin{scriptsize}
\begin{equation}\nonumber
\footnotesize
\max \frac{\sum\limits_{i=1}^c Tr(U_p^T(\overline{X_i}-\overline{X})V_p)^2}{\sum\limits_{i=1}^c \sum\limits_{X_j \in \pi_i}Tr(U_p^T(X_j-\overline{X_i})V_p)^2 + \lambda Tr(U_pV_p^TV_pU_p^T)}
\end{equation}
\begin{equation}\nonumber
\text{s.t.}~~Tr(U_p^TU_pV_p^TV_p)=1
\end{equation}
\end{scriptsize}

Since there is no straightforward solution to this function, we aim to rewrite the objective function to a generalized eigen-decomposition problem. According to Lemma \ref{lemma4} and Lemma \ref{lemma6} in the appendix, we can rewrite the objective function as follows:
\begin{scriptsize}
\begin{equation}\nonumber
\footnotesize
 \max \frac{\sum\limits_{i=1}^c(u_p^T(I \otimes (\overline{X_i}-\overline{X}))v_p)^2}{\sum\limits_{i=1}^c \sum\limits_{X_j \in \pi_i}(u_p^T (I \otimes (X_j- \overline{X_i}))v_p)^2 + \lambda Tr(U_pV_p^TV_pU_p^T)}
\end{equation}
\begin{equation}\nonumber
\text{s.t.}~~u_p^TD_p^vu_p=1
\end{equation}
\end{scriptsize}

\begin{equation}
\Leftrightarrow \max \frac{u_p^TM_p^vu_p}{u_p^TN_p^vu_p},
\label{eq_v}
\end{equation}
\begin{equation}\nonumber
\text{s.t.}~~u_p^TD_p^vu_p=1
\end{equation}
where 
\begin{equation}
u_p = vec(U_p)
\label{value_uki}
\end{equation}

\begin{equation}
v_p = vec(V_p)
\label{value_vki}
\end{equation}

\begin{equation}
D_p^v = (V_p^TV_p) \otimes I
\label{value_dkv}
\end{equation}

\begin{equation}
M_p^v = \sum_{i=1}^c (I \otimes (\overline{X_i}-\overline{X}))v_pv_p^T(I \otimes (\overline{X_i}-\overline{X})^T)
\label{value_mkv}
\end{equation}

\begin{small}
\begin{equation}
N_p^v = \sum_{i=1}^c \sum_{X_j \in \pi_i} (I \otimes (X_j- \overline{X_i}))v_pv_p^T(I \otimes(X_j-\overline{X_i})^T) + \lambda D_p^v
\label{value_nkv}
\end{equation}
\end{small}

It is noticed that the rewritten objective function in \eqref{eq_v} has become similar to the optimization problem in \eqref{opt_LDA}. Therefore, we can compute the optimal $U_p$ by solving the optimization problem in \eqref{eq_v} as follows:
\begin{equation}
u_p= \frac{q}{\sqrt{q^TD_p^vq}},
\label{value_ukv}
\end{equation} where $q$ is the largest eigenvector of $(N_p^v)^{-1}M_p^v$.

Next, we compute the optimal $V_p$ for a fixed $U_p$.
\begin{scriptsize}
\begin{equation}\nonumber
\max \frac{\sum\limits_{i=1}^c Tr(U_p^T(\overline{X_i}-\overline{X})V_p)^2}{\sum\limits_{i=1}^c \sum\limits_{X_j \in \pi_i} Tr(U_p^T(X_j-\overline{X_i})V_p)^2 + \lambda Tr(U_pV_p^TV_pU_p^T)}
\end{equation}

\begin{equation}\nonumber
\text{s.t.}~~Tr(U_p^TU_pV_p^TV_p)=1
\end{equation}
\end{scriptsize}

With similar procedures, we rewrite the objective function to an eigen-decomposition problem.
According to Lemma \ref{lemma4} and Lemma \ref{lemma6} in Appendix, the objective function can be transformed into:
\begin{scriptsize}
\begin{equation}
\nonumber
\max \frac{\sum\limits_{i=1}^c (v_p^T(I \otimes (\overline{X_i}-\overline{X})^T)u_p)^2}{\sum\limits_{i=1}^c \sum\limits_{X_j \in \pi_i} (v_p^T(I \otimes (X_j - \overline{X_i})^T)u_p)^2 + \lambda Tr(U_pV_p^TV_pU_p^T)}
\end{equation}
\end{scriptsize}
\begin{equation}\nonumber
\text{s.t.}~~v_p^TD_p^uv_p=1
\end{equation}

\begin{equation}
\Leftrightarrow \max \frac{v_p^TM_p^uv_p}{v_p^TN_p^uv_p},
\label{eq_u}
\end{equation}
\begin{equation}\nonumber
\text{s.t.}~~v_p^TD_p^uv_p=1
\end{equation}
where
\begin{equation}
D_p^u=(U_p^TU_p) \otimes I      
\label{value_dku}                                                      
\end{equation}

\begin{equation}
M_p^u = \sum_{i=1}^c (I \otimes (\overline{X_i}-\overline{X})^T)u_pu_p^T(I \otimes (\overline{X_i}-\overline{X}))
\label{value_mku}
\end{equation}

\begin{small}
\begin{equation}
N_p^u = \sum_{i=1}^c \sum_{X_j \in \pi_i} (I \otimes (X_j-\overline{X_i})^T)u_pu_p^T(I \otimes (X_j - \overline{X_i})) + \lambda D_p^u
\label{value_nku}
\end{equation}
\end{small}

Similarly, we can compute the optimal $V_p$ by solving the optimization problem in Eq. \eqref{eq_u}. The solution is 
\begin{equation}
v_p = \frac{q}{\sqrt{q^TD_p^uq}},
\label{value_vk}
\end{equation} where $q$ is the largest eigenvector of $(N_p^v)^{-1}M_p^v$.

The optimizations of $U_p|_{p=1}^h$ and $V_p|_{p=1}^h$ are iterated until convergence. Pseudo-code for our proposed CRP is given in Algorithm $1$. We set $k=2$ empirically. The most time-consuming steps are eigenvalue decomposition operations (line 10 and line 16 in Algorithm 1) and the total time complexity is $O((k\max(l_1,l_2))^3)$. For example, we use Coil20 dataset, in which the dimensionality of the original data is $32 \times 32$. The time complexity of CRP is $O((2 \times 32)^3)$, which is significantly small compared to that of the classical LDA \cite{bib_Ye} $O((1024)^3)$.

\begin{algorithm}
\caption{Optimization Algorithm for CRP}
\SetAlgoLined
\KwData{$\textbf{X} = \{X_1, X_2, ..., X_n\}, X_i\in \mathbb{R}^{l_1 \times l_2}$}
\KwResult{$U_p|_{p=1}^h$ and $V_p|_{p=1}^h$}
\For{$p \leftarrow 1$ \KwTo $h$}{
\emph{Compute the mean $\overline{X_i}$ of the $i$th class for each $i$}\;
\emph{Compute the global mean $\overline{X}$}\;
\emph{Initialise $V_p$ as I(c,k)}\;
\Repeat{Convergence}{
\emph{Update $v_p$ using Eq. \eqref{value_vki}}\;
\emph{Update $D_p^v$ using Eq. \eqref{value_dkv}}\;
\emph{Update $M_p^v$ using Eq. \eqref{value_mkv}}\;
\emph{Update $N_p^v$ using Eq. \eqref{value_nkv}}\;
\emph{p=the largest eigenvector of $(N_p^v)^{-1}M_p^v$}\;
\emph{Update $u_p$ using Eq. \eqref{value_ukv}}\;
\emph{$U_p=\text{reshape}(u_p,r,k)$}\;
\emph{Update $D_p^u$ using Eq. \eqref{value_dku}}\;
\emph{Update $M_p^u$ using Eq. \eqref{value_mku}}\;
\emph{Update $N_p^u$ using Eq. \eqref{value_nku}}\;
\emph{p=the largest eigenvector of $(N_p^u)^{-1}M_p^u$}\;
\emph{Update $v_p$ using Eq. \eqref{value_vk}}\;
\emph{$V_p=\text{reshape}(v_p, c, k)$}\;
}
\emph{Update the training data according to Eq. \eqref{eq_x}}
}
\end{algorithm}

\section{convergence analysis}

In this section, we show that Algorithm 1 converges monotonically and thus we can obtain the local optima of $U_j|_{j=1}^h$ and $V_j|_{j=1}^h$. We prove that Algorithm 1 converges by the following theorem.

\begin{theorem}
\label{theorem7}
The value of objective function $f$ of our proposed algorithm promises to increase monotonically until convergence.
\end{theorem}

\begin{proof}
Suppose we have finished the first $r$-th iteration, we have got $U_j^r|_{j=1}^h$ and $V_j^r|_{j=1}^h$. Now we continue the next iteration. For a fixed ${V_j}|_{j=1}^h$ as ${V_j}^r|_{j=1}^h$ we solve the optimization problem for $U_j^{r+1}|_{j=1}^h$. For a fixed $V_j|_{j=1}^h$, the objective function in Eq. \eqref{eq_opt} is converted to the optimization problem in Eq. \eqref{eq_v}. We can easily tell that it is a convex optimization problem $w.r.t~U_j|_{j=1}^h$. Hence, the optimal solution for $U_j|_{j=1}^h$ can be obtained by setting the derivative of Eq.\eqref{eq_v} $w.r.t$ $U_j|_{j=1}^h$ to zero respectively. Thus, we have:

\begin{equation}\label{eq_17}
\scriptsize
\begin{aligned}
& \frac{\sum\limits_{i=1}^c Tr({U_p^{r+1}}^T(\overline{X_i}-\overline{X})V_p^r)^2}{\sum_{i=1}^c \sum\limits_{X_j \in \pi_i} Tr({U_p^{r+1}}^T(X_j-\overline{X_i})V_p^r)^2 + \lambda Tr(U_p^{r+1}{V_p^r}^TV_p^r{U_p^{r+1}}^T)} \\
& \geq \frac{\sum\limits_{i=1}^c Tr({U_p^{r}}^T(\overline{X_i}-\overline{X})V_p^r)^2}{\sum\limits_{i=1}^c \sum\limits_{X_j \in \pi_i} Tr({U_p^{r}}^T(X_j-\overline{X_i})V_p^r)^2 + \lambda Tr(U_p^r{V_p^r}^TV_p^r{U_p^r}^T)}
\end{aligned}
\end{equation}

Similarly, we can obtain the following inequality for fixed $U_j|_{j=1}^h$ as $U_j^r|_{j=1}^h$.

\begin{equation}\label{eq_18}
\scriptsize
\begin{aligned}
& \frac{\sum_{i=1}^c Tr({U_p^r}^T(\overline{X_i}-\overline{X})V_p^{r+1})^2}{\sum_{i=1}^c \sum\limits_{j=1}^{n_i} Tr({U_p^r}^T(x_j-\overline{X_i}){V_p^{r+1}})^2 + \lambda Tr(U_p^r{V_p^{r+1}}^T{V_p^{r+1}}{U_p^r}^T)} \\
& \geq \frac{\sum\limits_{i=1}^c Tr({U_p^{r}}^T(\overline{X_i}-\overline{X})V_p^r)^2}{\sum\limits_{i=1}^c \sum\limits_{X_j \in \pi_i} Tr({U_p^r}^T(X_j-\overline{X_i})V_p^{r})^2 + \lambda Tr(U_p^r{V_p^r}^TV_p^r{U_p^r}^T)}
\end{aligned}
\end{equation}

We have the following inequality by integrating Eq. \eqref{eq_17} and Eq. \eqref{eq_18}.

\begin{equation}\label{eq_19}
\scriptsize
\begin{aligned}
& \frac{\sum\limits_{i=1}^c Tr({U_p^{r+1}}^T(\overline{X_i}-\overline{X})V_p^{r+1})^2}{\sum\limits_{i=1}^c \sum\limits_{X_j \in \pi_i} Tr({U_p^{r+1}}^T(X_j-\overline{X_i})V_p^{r+1})^2 + \lambda Tr(U_p^{r+1}{V_r^{r+1}}^TV_r^{r+1}{U_r^{r+1}}^T)} \\
& \geq \frac{\sum\limits_{i=1}^c Tr({U_p^r}^T(\overline{X_i}-\overline{X})V_p^r)^2}{\sum_{i=1}^c \sum\limits_{X_j \in \pi_i} Tr({U_p^r}^T(X_j-\overline{X_i})V_p^{r})^2 + \lambda Tr(U_p^r{V_p^r}^TV_p^r{U_p^r}^T)}
\end{aligned}
\end{equation}

From Eq. \eqref{eq_19} we can see that the objective function value increases monotonically. Theorem \ref{theorem7} has been proved.

\end{proof}

\section{Experiment}
In this section, we test the proposed CRP algorithm. We compare CRP with seven algorithms, including LDA \cite{bib_Belhumeur}, 2DPCA \cite{bib_Jian}, 2DLDA \cite{bib_Ye}, BilinearSVM \cite{bib_Pirsiavash}, two non-iterative 2DLDA algorithms (S2DLDA and P2DLDA) \cite{bib_Inoue} and Tensor LPP \cite{bib_He_Ten}.

There are five parts in our experiments. We first validate how fast our algorithm converges over the five different datasets. Secondly, we evaluate how the classification performance varies w.r.t different $k$s. Thirdly, we conduct several different initializations and report the performance variance with different initialization manners. Then, we compare the results of classification in a variety of multimedia analysis, including face recognition, object recognition, facial expression recognition, head pose recognition and handwritten digit recognition. Accuracy is used as the evaluation metric for classification. Finally, comparisons have been also made under a two-class setting, in which gender recognition is performed over three different face datasets.

Following \cite{bib_Ye,bib_Li,bib_Ma}, we use the gray pixel values of the images as the features. In all of our experiments, we randomly sample 3, 5, 10, and 20 data per class as the training data for all the experiments. The remaining samples are used as testing data. To evaluate the performance with sufficient number of training data, we further use 80\% data as training data and the remaining as testing data. The regularization parameter, $\lambda$, in the proposed algorithm is tuned in a range of $\{10^{-6}, 10^{-4}, \dots , 10^4, 10^6\}$ and the best result is reported. We independently repeat the experiments five times and report the results of average accuracy with the stand deviations. Following the work in \cite{bib_Belhumeur,HouNZYW14}, we project the original data into a $(c-1)^2$ dimensional subspace for all the compared algorithms. LIBSVM is applied as the implementation of SVM. We learn the optimal regularization parameter of SVM through a tenfold cross-validation. 

\begin{table}
\label{setting}
\renewcommand{\arraystretch}{1.3}
\caption{Dataset details}
\centering
\begin{tabular}{|c||c|c|c|}
\hline
\bfseries Dataset & \bfseries Matrix Size & \bfseries Dataset Size & \bfseries Class \#\\
\hline \hline

UUIm & $24 \times 32$ & 2,220 & 10\\
\hline
CVL & $32 \times 32$ & 21,780 & 10\\
\hline
Pointing'04(Tilt) & $40 \times 30$ & 2,790 & 9\\
\hline
Pointing'04(Pan) & $40 \times 30$ & 2,790 & 13\\
\hline 
USPS & $16 \times 16$ & 9,298 & 10 \\
\hline
COIL-20 & $32 \times 32$ & 1,440 & 20\\
\hline
\end{tabular}
\end{table}

\subsection{Datasets Description}

UUIm: The UUIm Head Pose and Gaze database \cite{bib_WeidenBacher} is used to evaluate the performance of head pose and gaze. This database comprises 2,220 images from ten different people. In our experiment, we resize each image to $24 \times 32$. In this database, all horizontal head poses are with a vertical orientation of 0 degree.

CVL: The CVL dataset \cite{bib_Diem} is used to evaluate the performance of handwritten digit recognition. There are 21,780 handwritten digit images in this dataset. In our experiment, we resize each image to $32 \times 32$.

Pointing'04: The Pointing'04 dataset \cite{bib_Gourier} is used for head pose estimation. Pointing'04 comprises 2,790 images from 15 people. In our experiment, we resize each image to $40 \times 30$. Both the tilt and pan angles are used to determine the head pose. For tilt, there are nine poses. Whereas for pan, there are 13 poses.

USPS: We use the USPS database to test the performance of our algorithm on handwritten digit recognition. There are 9,298 handwritten digit images in this database. We resize all the images to $16 \times 16$.

Coil20: Coil20 comprises 1,440 images of 20 objects. In our experiment, we resize each image to $32 \times 32$.

The detailed information of the dataset is summarized in Tab. 1, including the number of samples, the feature dimensions and the total number of classes.

\begin{figure*}[!ht]
\centering
\subfigure[]{
\includegraphics[width=0.15\linewidth]{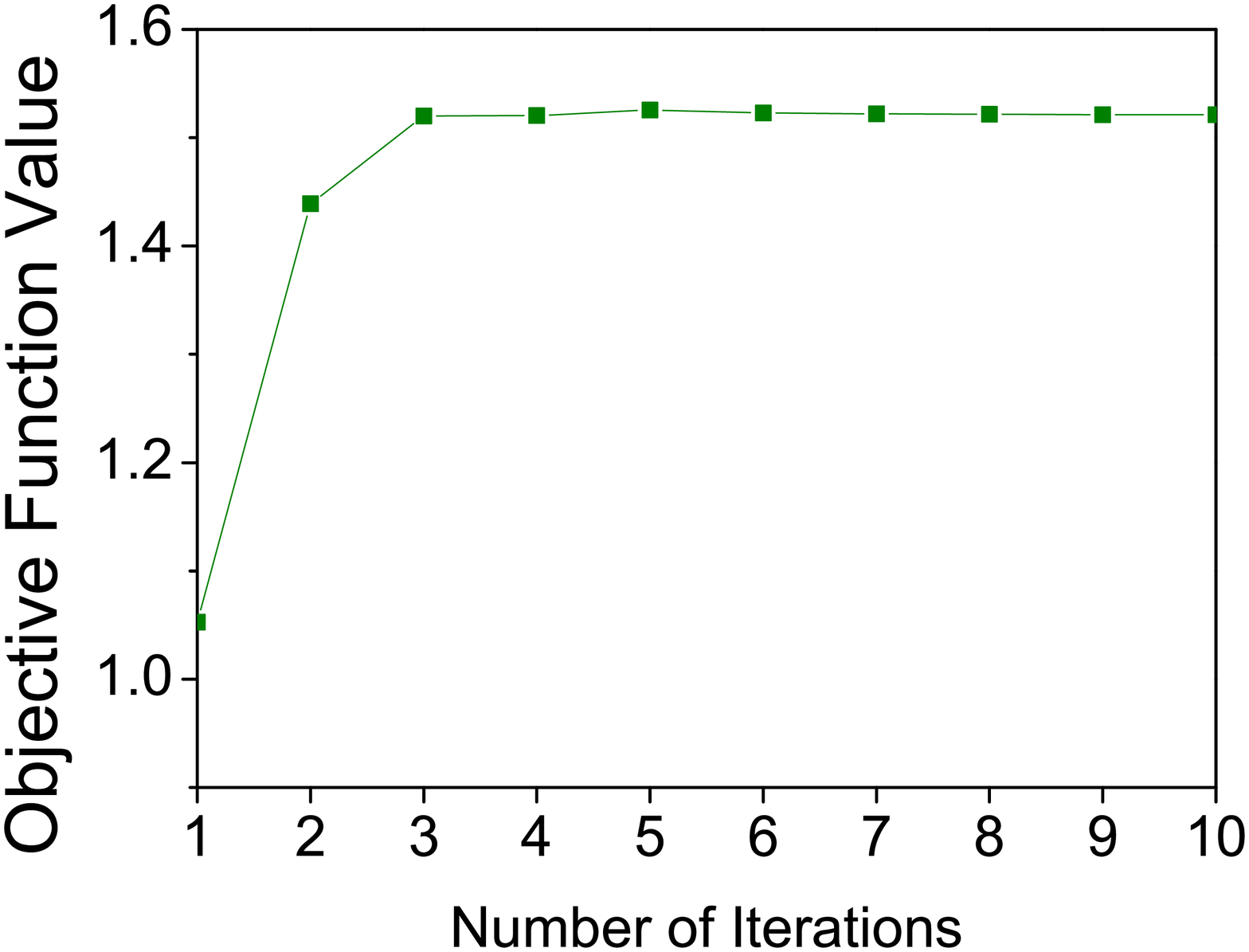}}
\subfigure[]{
\includegraphics[width=0.15\linewidth]{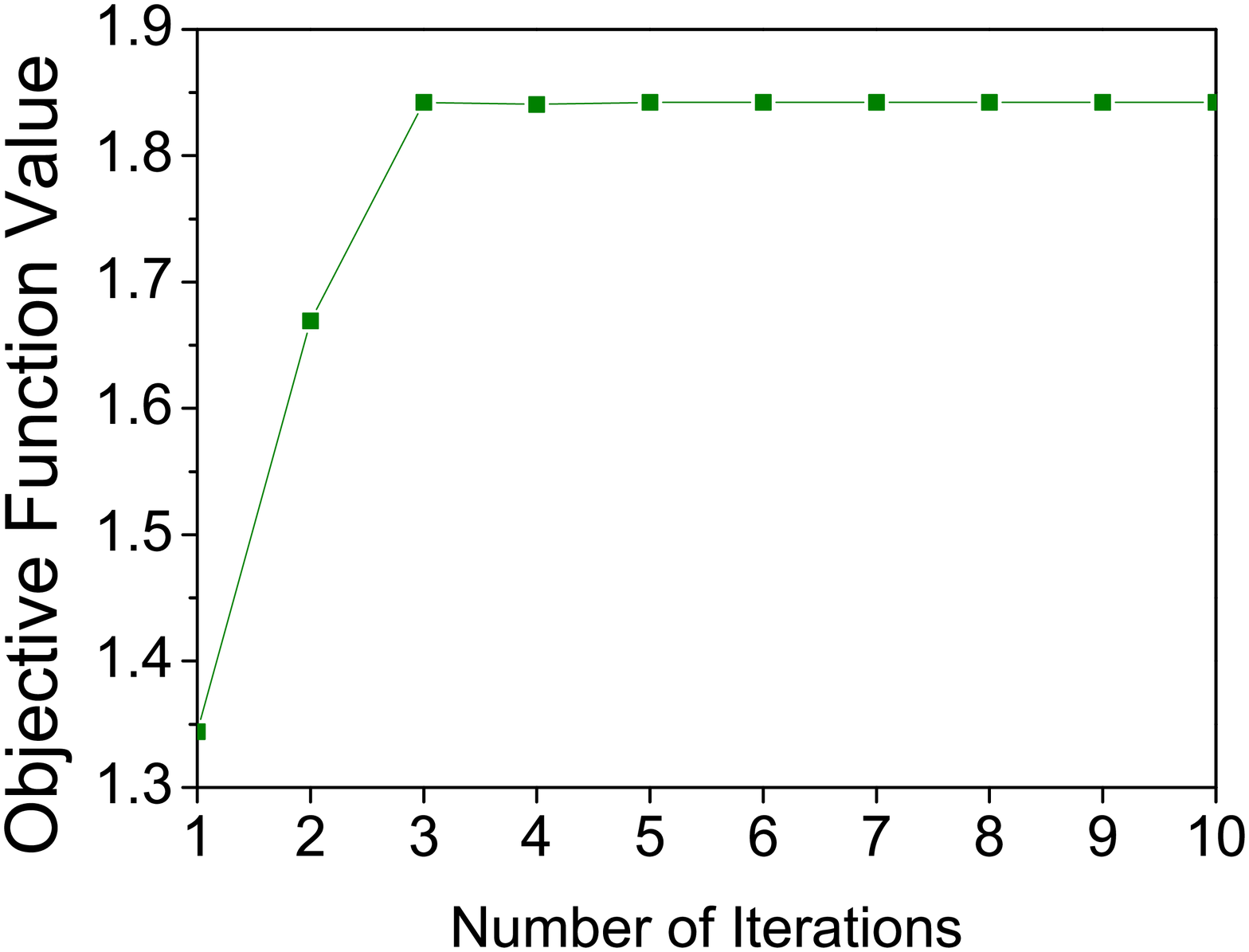}}
\subfigure[]{
\includegraphics[width=0.15\linewidth]{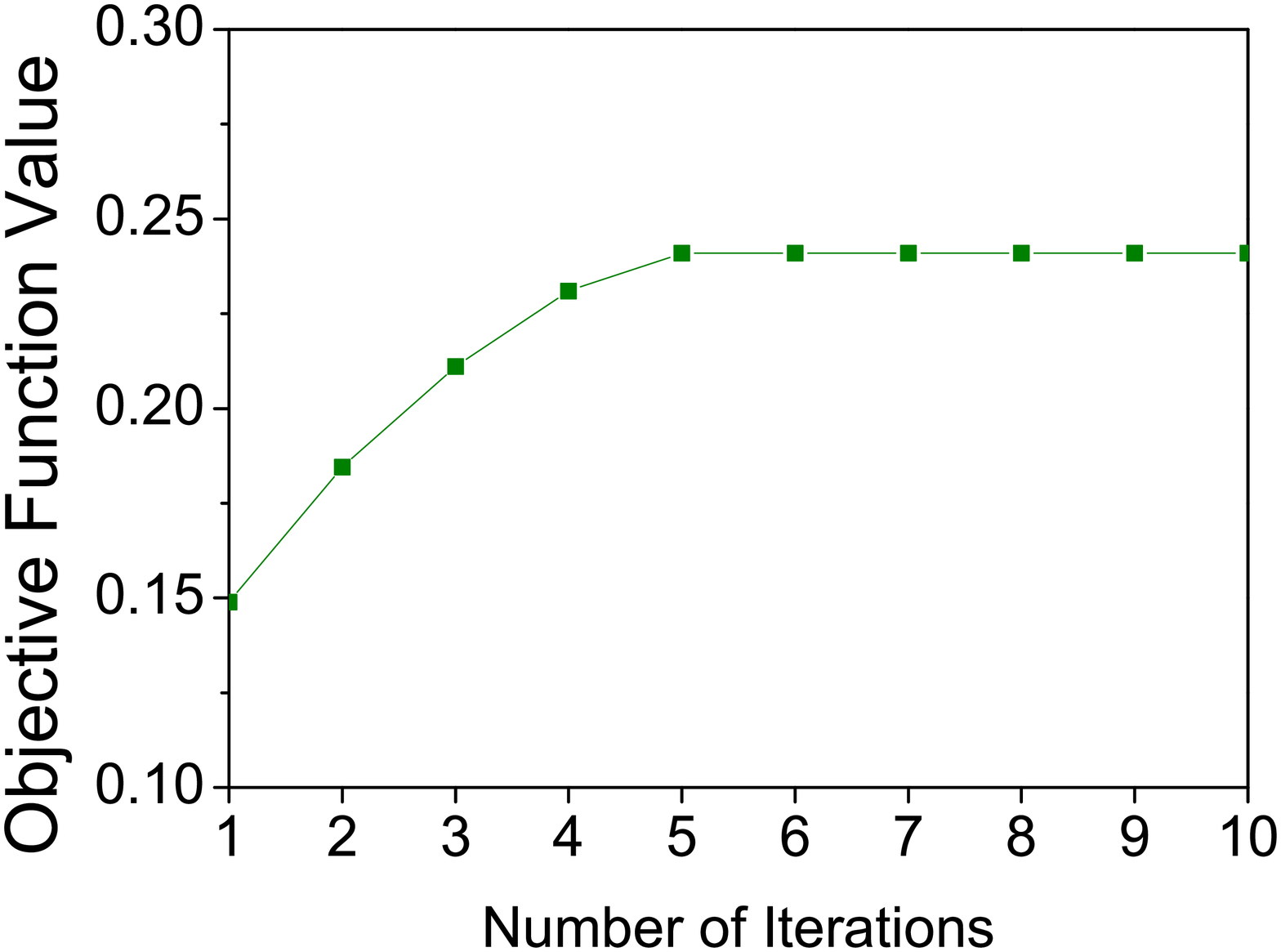}}
\subfigure[]{
\includegraphics[width=0.15\linewidth]{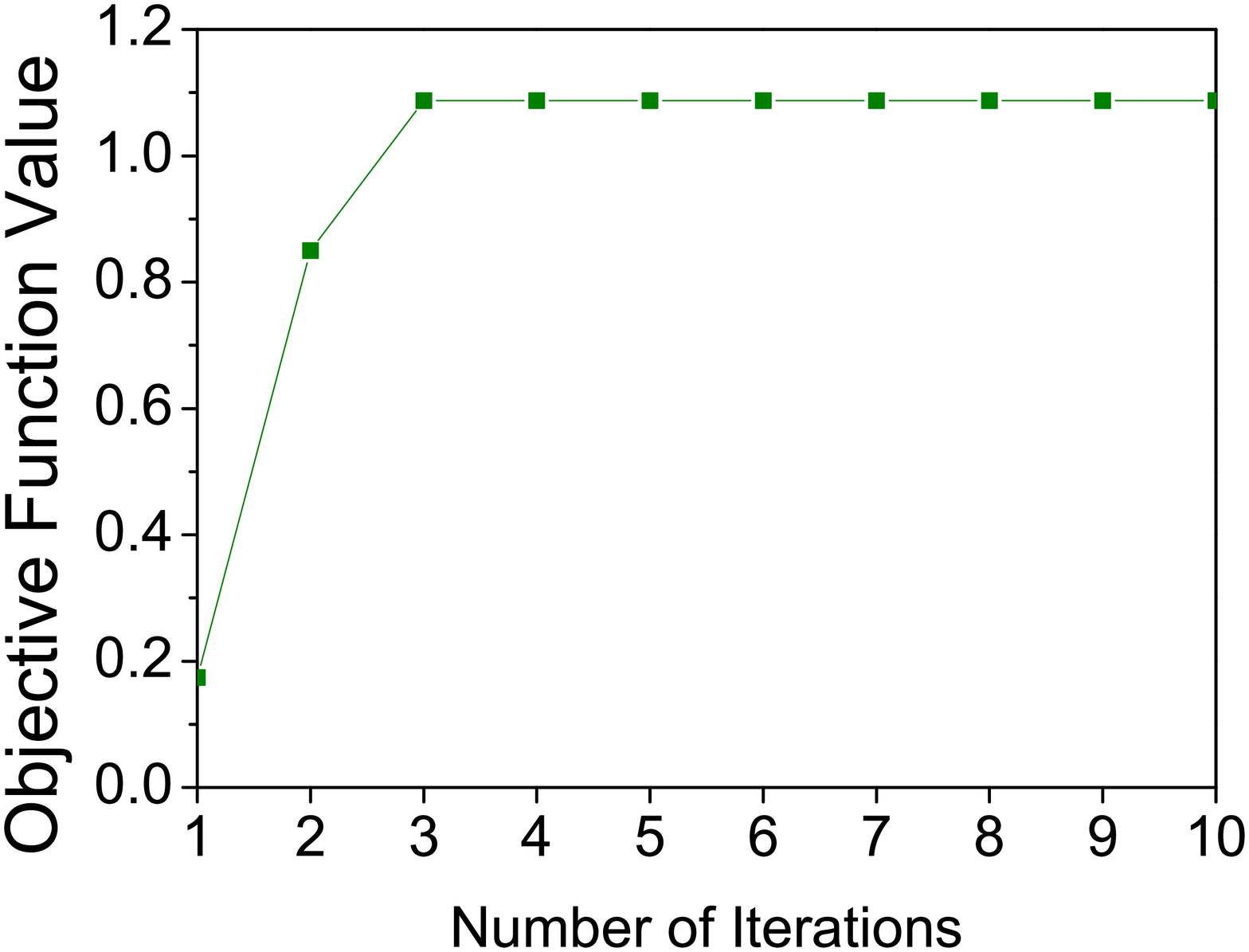}}
\subfigure[]{
\includegraphics[width=0.15\linewidth]{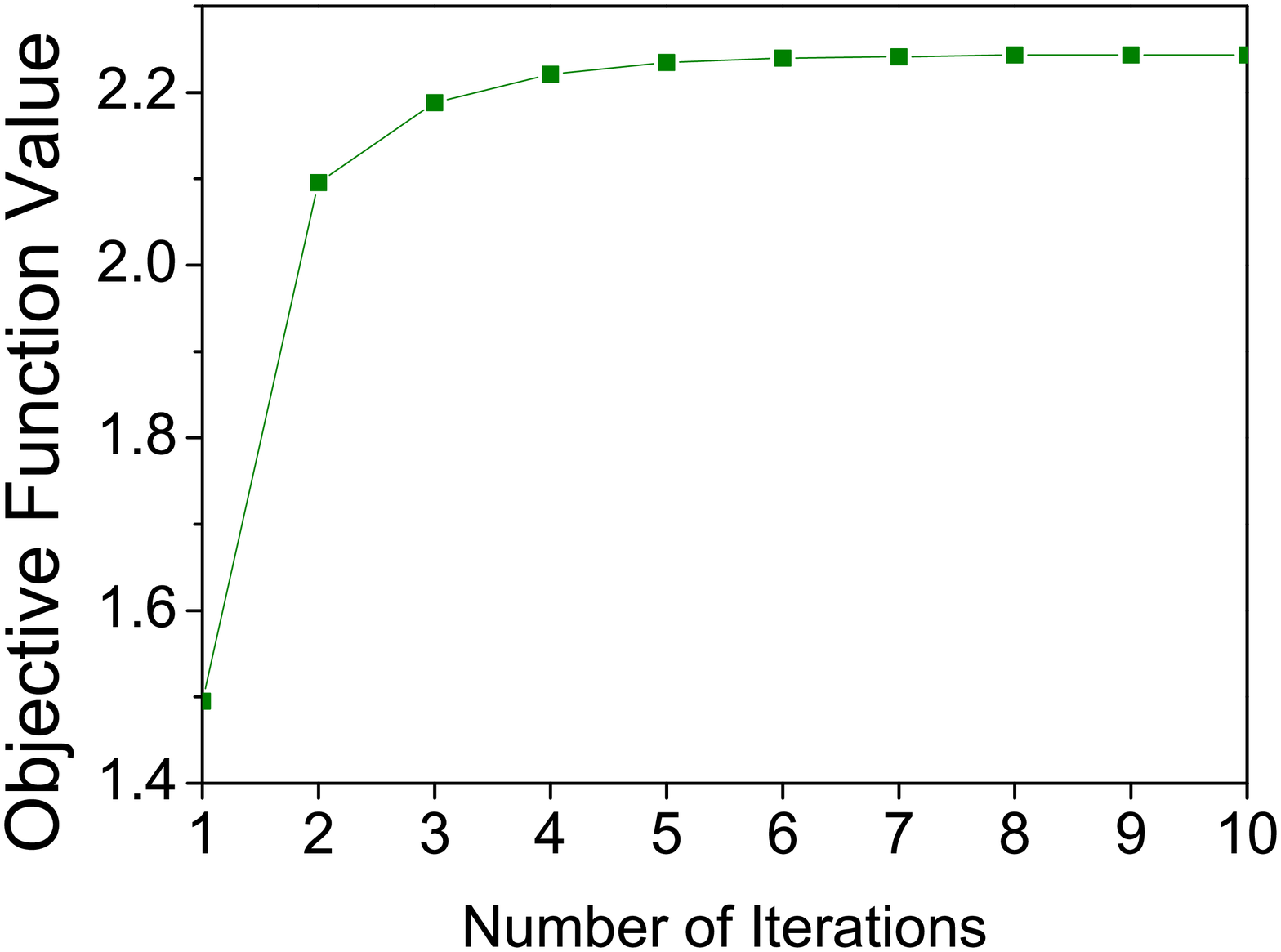}}
\subfigure[]{
\includegraphics[width=0.15\linewidth]{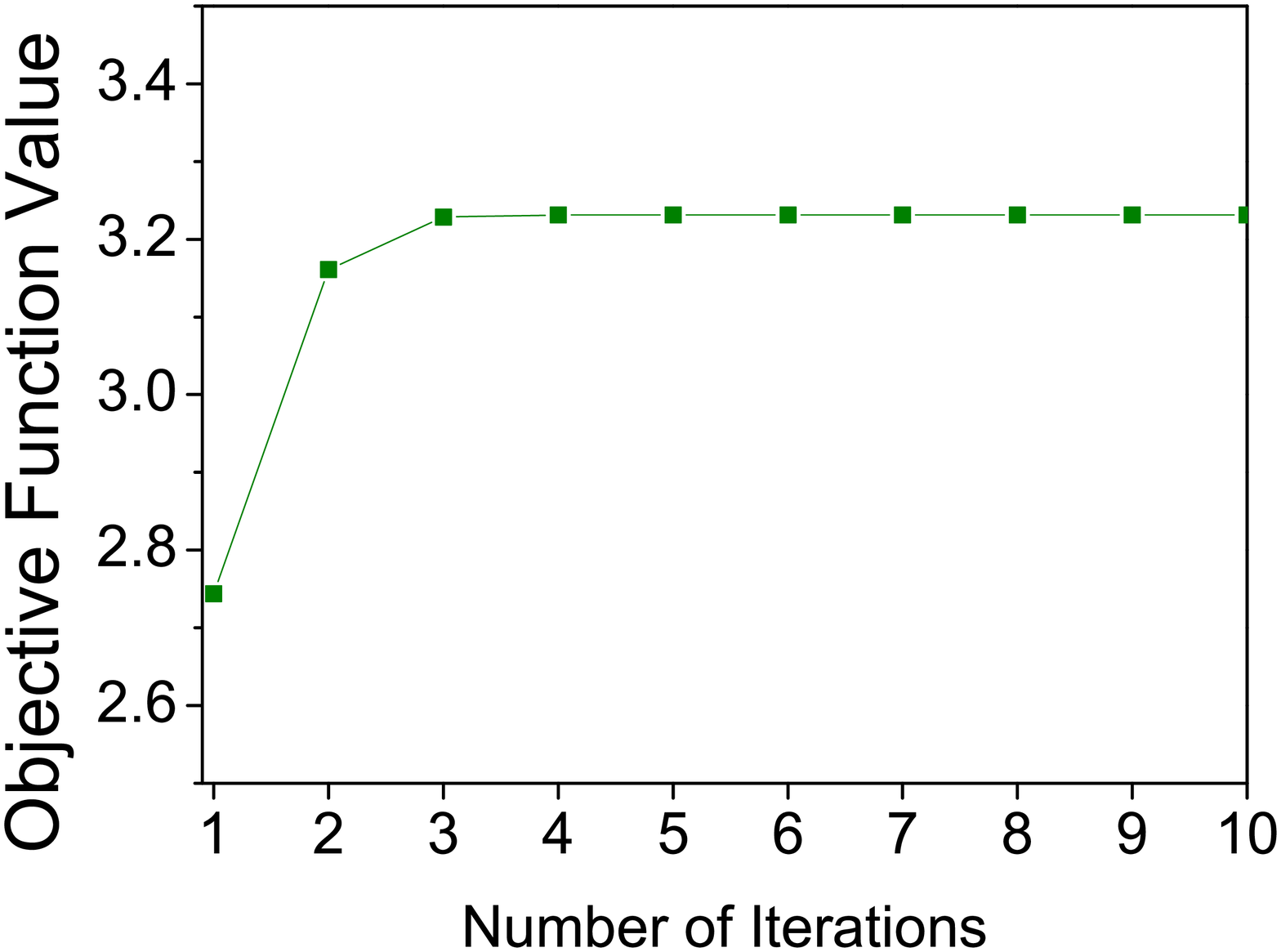}}
\caption{Convergence demonstration on different datasets. (a) UUIm, (b) CVL, (c) Pointing'04(tilt), (d) Pointing'04(pan), (e) USPS, (f) Coil20.} \label{Fig1}
\end{figure*}

\subsection{Experimental Results}

\subsubsection{Convergence Demonstration}
We first conduct experiments to validate the convergence of our algorithm. Note that our algorithm learns multiple (i.e., $k$) projection models one by one via the same approach. We therefore randomly select one projection model and plot its objective function values. The convergence demonstration on five different datasets is shown in Fig. \ref{Fig1}, where the vertical axis stands for the objective function value and the horizontal axis denotes the number of iterations. From Fig. \ref{Fig1}, we observe that our proposed algorithm converges fast on all datasets. In most of the cases, the algorithm converges within ten iterations, demonstrating that the proposed optimization algorithm efficiently converges. For other projection models, we observe similar results.

\subsubsection{Performance Variance w.r.t $k$}
In this section, experiments are conducted to study how $k$ affects the performance of the proposed algorithm. UUIm dataset is utilized in this experiment. Three data per class are utilized as training data.

\begin{figure*}[!ht]
\centering
\subfigure[]{
\includegraphics[width=0.15\linewidth]{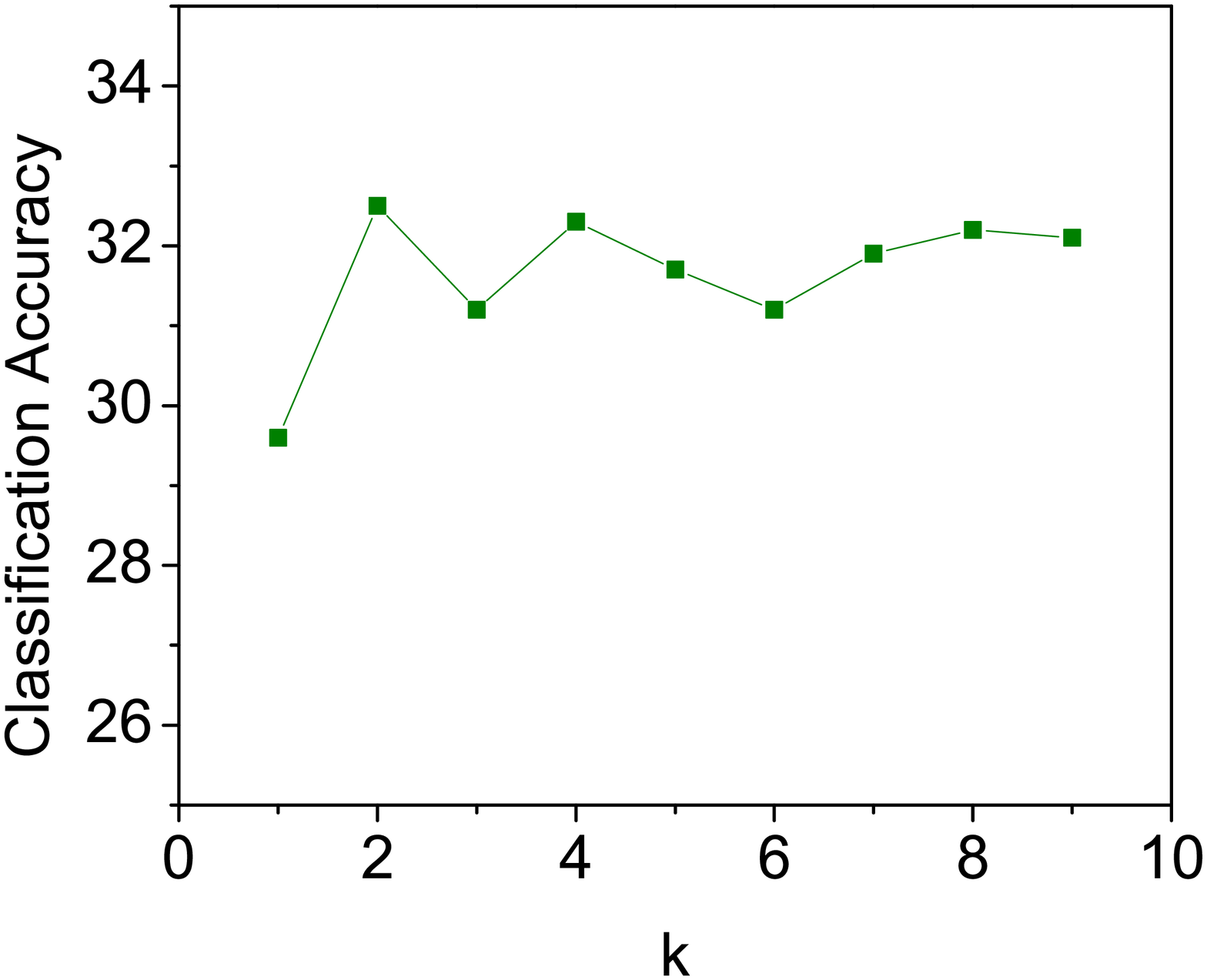}}
\subfigure[]{
\includegraphics[width=0.15\linewidth]{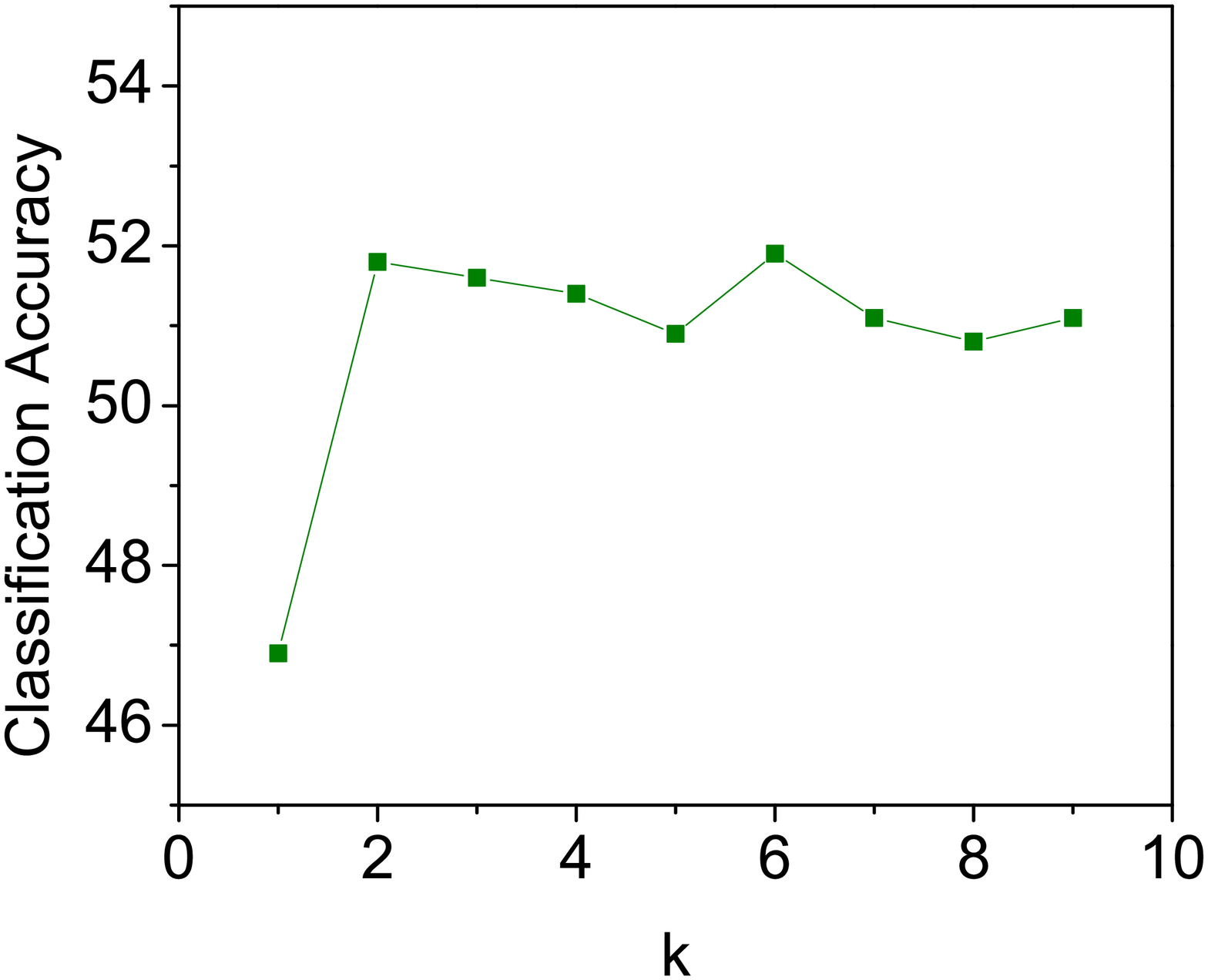}}
\subfigure[]{
\includegraphics[width=0.15\linewidth]{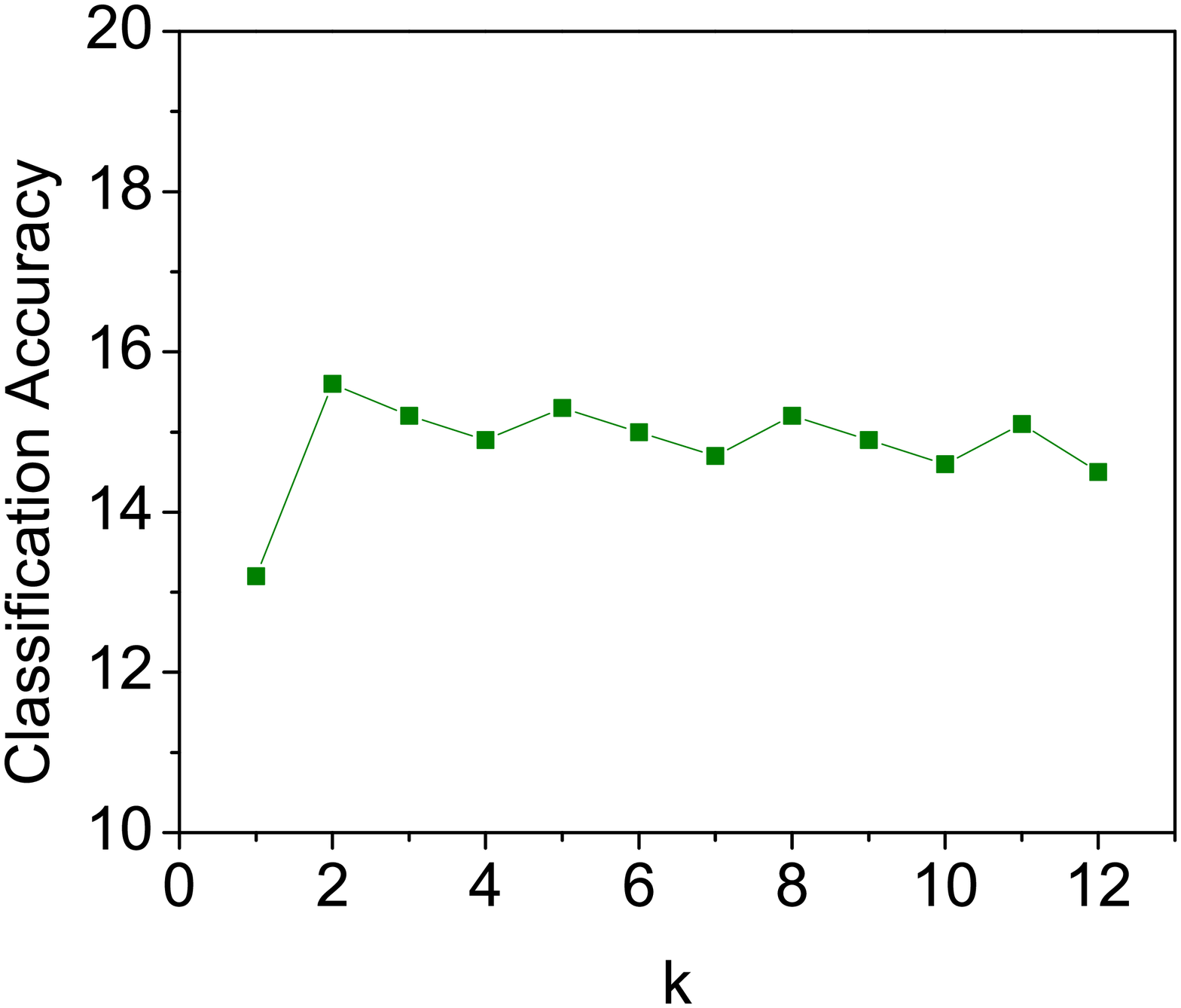}}
\subfigure[]{
\includegraphics[width=0.15\linewidth]{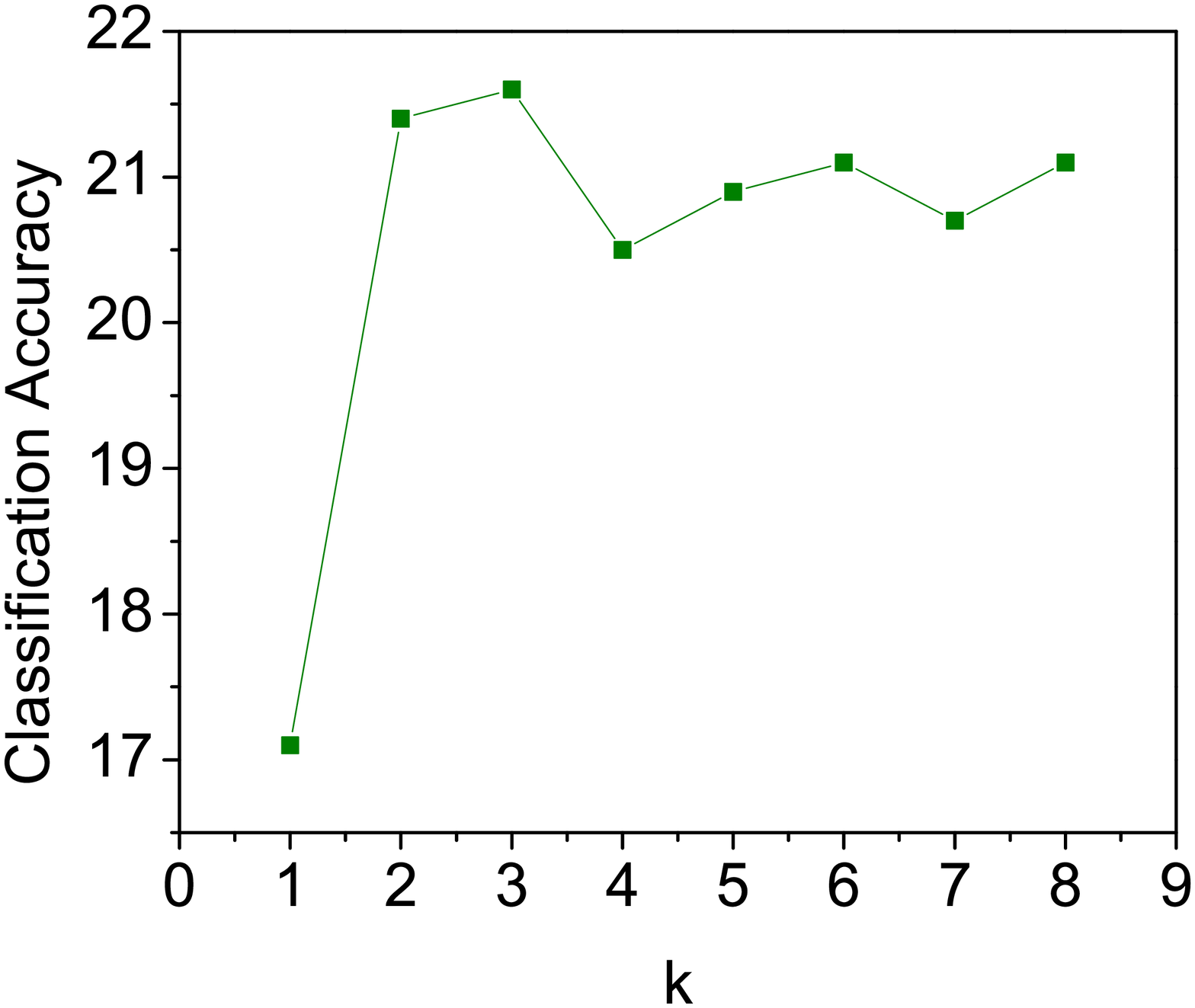}}
\subfigure[]{
\includegraphics[width=0.15\linewidth]{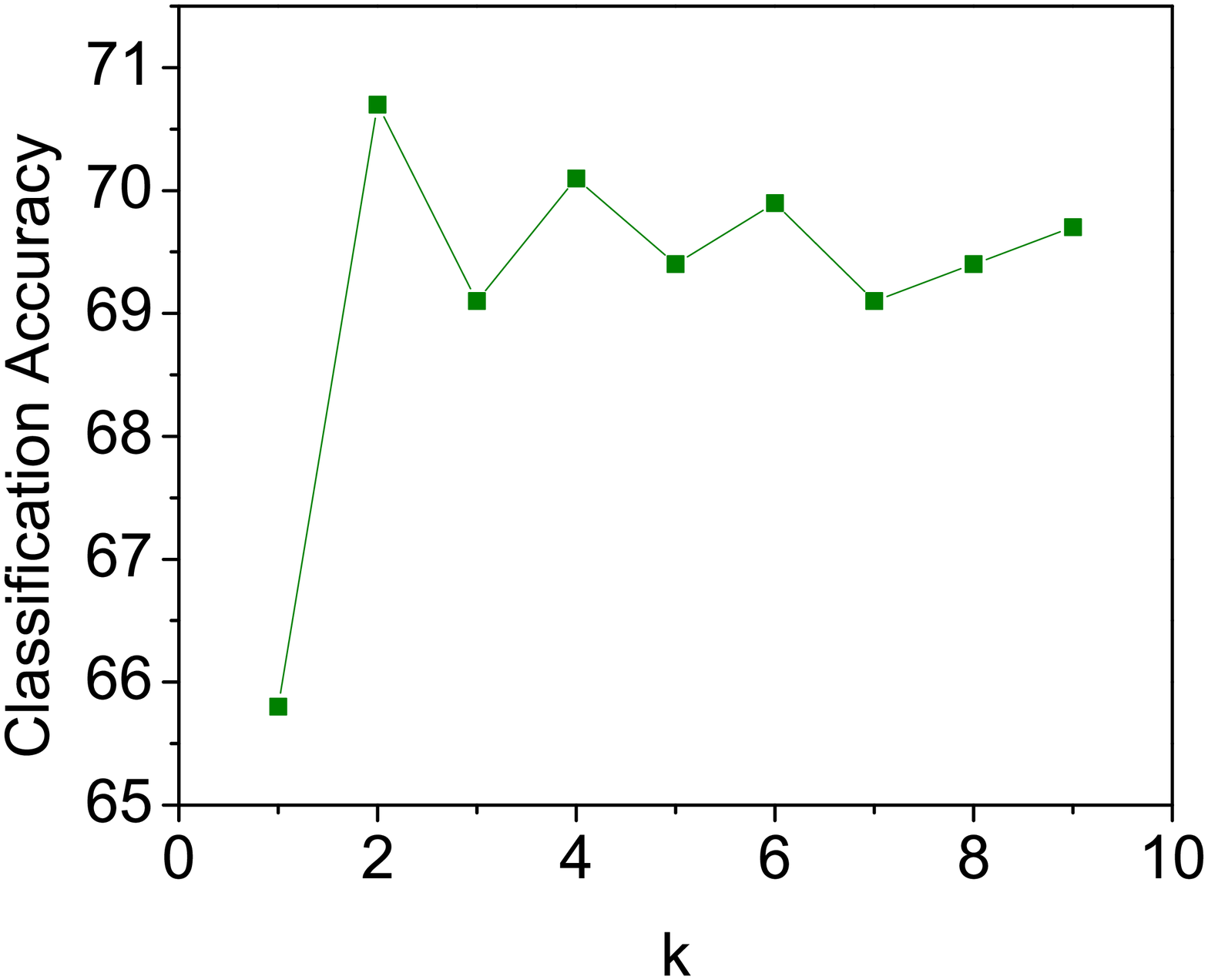}}
\subfigure[]{
\includegraphics[width=0.15\linewidth]{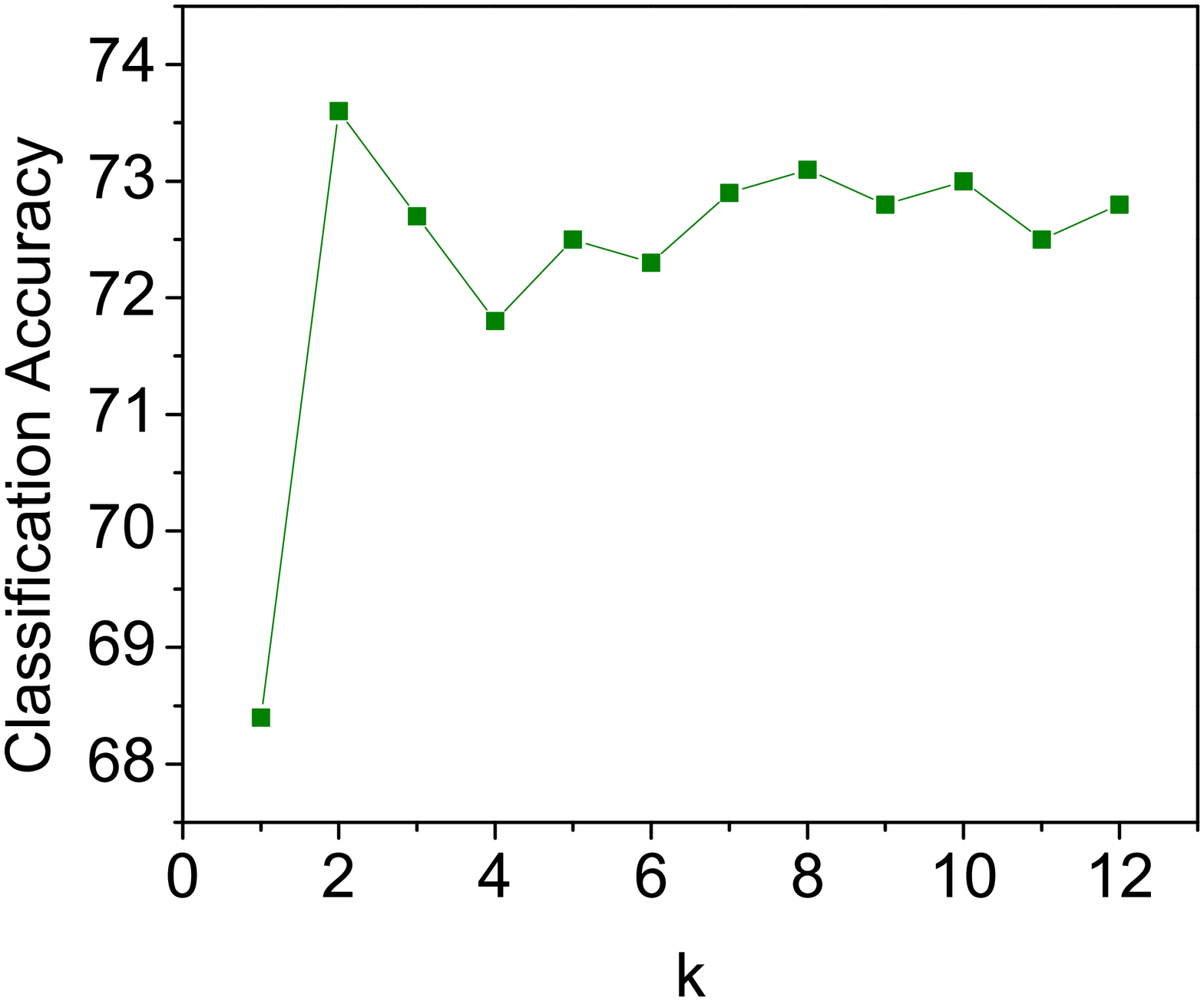}}
\caption{Performance Variance w.r.t $k$ on different datasets. (a) UUIm, (b) CVL, (c) Pointing'04(tilt), (d) Pointing'04(pan), (e) USPS, (f) Coil20.} \label{Fig2}
\end{figure*}

Fig. \ref{Fig2} shows classification accuracy varies when changing different $k$s. Taking UUIm as an example, we have the following observations: 1) When $k$ is set to 1, the classification accuracy is relatively low, at only 29.6\%. 2) When we increase $k$ to 2, the classification accuracy rises to about 32.5\%. 3) When $k$ is further increased, the classification accuracy is stable.

Based on the above observations, we empirically set $k$ to 2 in the remaining experiments, which can decrease computation complexity and obtain decent results as well.

\subsubsection{Performance with Different Initializations}

In this section, experiments are conducted to evaluate how performance varies with different initializations. We select three samples per class for training and the rest as testing data. We conduct different initializations, including setting all the diagonal elements of $V$ to 0.5, 1, 2 respectively and random values. The experimental results are shown in Tab. \ref{differentIni}.

From the experiment results, we can observe that the proposed algorithm can always get good local optima with different initializations.

\begin{table}[!ht]
\small
\label{initialization}
\renewcommand{\arraystretch}{1.3}
\caption{Performance w.r.t Different Initializations}
\scriptsize
\centering
\begin{tabular}{|c||c|c|c|c|}
\hline
\bfseries Dataset & \bfseries 0.5 & \bfseries 1 & \bfseries 2 & \bfseries random  \\
\hline \hline
UUIm & $32.3 \pm 1.4$ & $33.1 \pm 1.1$ & $32.8 \pm 1.2$ & $32.4 \pm 1.4$ \\
\hline
CVL  & $52.1 \pm 1.8$ & $52.8 \pm 1.5$ & $51.9 \pm 1.6$ & $51.9 \pm 1.5$ \\
\hline
Pointing'04(Tilt) & $20.9 \pm 1.5$ & $21.7 \pm 1.7$ & $21.0 \pm 1.4$ & $21.4 \pm 1.2$ \\
\hline
Pointing'04(Pan) &  $15.1 \pm 0.8$ & $15.8 \pm 1.1$ & $16.1 \pm 1.3$ & $15.6 \pm 1.4$ \\
\hline
USPS   & $70.1 \pm 1.1$ & $70.9 \pm 1.5$ & $70.5 \pm 1.4$ & $70.7 \pm 1.3$ \\
\hline
Coil20  & $73.1 \pm 1.4$ & $73.6 \pm 1.7$ & $73.8 \pm 1.2$ & $73.6 \pm 1.4$ \\
\hline
\end{tabular}
\label{differentIni}
\end{table}

\subsubsection{Classification Performance Comparison}
In this experiment, we compare classification performance of our algorithm with other methods, including 2DPCA \cite{bib_Jian}, 2DLDA \cite{bib_Ye}, S2DLDA \cite{bib_Inoue}, P2DLDA \cite{bib_Inoue}, T-LPP \cite{bib_He_Ten}, Bilinear SVM \cite{bib_Pirsiavash}. Accuracy is used as an evaluation metric. Note that once we obtain a new lower dimensional representation using CRP, any classification algorithm can be used for data classification. In order to show the discriminant capability of our algorithm, we use two classifiers in this experiment. The first one is SVM, which has been widely used. The second one is $1$-Nearest-Neighbor (1NN), which is used in \cite{bib_Ye,bib_Li} to evaluate the effectiveness of the classical 2DLDA and other related algorithms. Since bilinear SVM is a classifier, we directly compare the classification results obtained from bilinear SVM with other methods. The results of classification performance using SVM are reported from Tab. \ref{table_result_3accSVM} to Tab. \ref{table_result_80percaccSVM} and the results of classification performance using 1NN are presented from Tab. \ref{table_result_3accNN} to Tab. \ref{table_result_80percaccNN}. When we use SVM as the classifier, we have the following observations.

CRP outperforms the other seven algorithms, which demonstrates that the utilization of multiple projection models is useful to improve the classification performance. We also observe that when there are insufficient training data, CRP has more advantages over the other compared algorithms. For example, when we use three training data per class for CVL, the classification accuracy of CRP outperforms 2DLDA by $107\%$, relatively. It further indicates that CRP is more capable of capturing discriminant information when the training samples are quite limited. S2DLDA gets slightly better performance results than 2DLDA, on the whole, demonstrating that selecting a projection model with a higher discriminant capability can improve the classification performance. Nevertheless, CRP still outperforms S2DLDA. 

With the increase of training samples, the classification results of all the compared algorithms are improved. CRP consistently performs better than the other compared algorithms. Take UUIm for an example. When we increase the number of training data from $3 \times c$ to $20 \times c$, the classification accuracy of Bilinear SVM, the second best algorithm, and CRP are improved from 28.2 to 61.2 and from 32.5 to 64.2 respectively. This observation indicates that the proposed algorithm adopts multiple projection models to enable more discriminant lower dimensional embeddings.

We can get similar trends when using $1$-NN as the classifier, again demonstrating the effectiveness of CRP.

To step further, we also evaluate the performance when the training samples are sufficient. 80\% samples are utilized for training. The experimental results are reported in Tab. \ref{table_result_80percaccSVM} and Tab. \ref{table_result_80percaccNN}. From the result, we can observe that when training samples are sufficient, the proposed algorithm still has the best performance with a relative improvement.

\begin{table*}[!ht]
\renewcommand{\arraystretch}{1.3}
\caption{\small Performance Comparison (Classification Accuracy $\pm$ std\%) of LDA, 2DPCA, 2DLDA, S2DLDA, P2DLDA, T-LPP, Bilinear SVM and CRP with three training data per class for each dataset. SVM is used as a classifier.}
\newcommand{\tabincell}[2]{\begin{tabular}{@{}#1@{}}#2\end{tabular}}
\centering
\begin{tabular}{|c||c|c|c|c|c|c|c|c|}
\hline
\bfseries Dataset & \bfseries LDA & \bfseries 2DPCA \cite{bib_Jian} & \bfseries 2DLDA \cite{bib_Ye}& \bfseries S2DLDA \cite{bib_Inoue}& \bfseries P2DLDA \cite{bib_Inoue}& \bfseries T-LPP \cite{bib_He_Ten}& \bfseries B-SVM \cite{bib_Pirsiavash}& \bfseries CRP \\
\hline \hline
UUIm &  $16.7 \pm 2.1$   & $20.1 \pm 1.8$  & $19.3 \pm 1.5$ & $26.4 \pm 1.2$ & $22.3 \pm 1.6$ & $26.5 \pm 1.4$ & $28.2 \pm 1.7$ & $\mathbf{32.5 \pm 1.3}$\\
\hline
CVL &    $21.8 \pm 1.2$  & $22.0 \pm 1.8$  & $24.3 \pm 1.5$ & $29.8 \pm 1.9$ & $22.8 \pm 1.4$ & $41.3 \pm 1.8$ & $47.1 \pm 1.6$ & $\mathbf{51.8 \pm 1.5}$\\
\hline
Pointing'04(Tilt) &  $6.9 \pm 0.8$  & $8.5 \pm 1.1$ & $10.3 \pm 1.4$ & $11.8 \pm 1.2$ & $9.3 \pm 1.6$ & $15.1 \pm 1.2$ & $18.1 \pm 1.7$ & $\mathbf{21.4 \pm 1.2}$ \\
\hline
Pointing'04(Pan) &  $6.4 \pm 1.3$   & $8.2 \pm 1.5$ & $9.9 \pm 1.6$ & $10.2 \pm 1.4$ & $8.1 \pm 1.0$ & $10.5 \pm 1.7$ & $13.1 \pm 1.5$ & $\mathbf{15.6 \pm 1.4}$ \\
\hline
USPS &    $42.1 \pm 1.7$   & $43.1 \pm 1.5$ & $45.8 \pm 1.8$ & $55.3 \pm 1.5$ & $41.9 \pm 1.3$ & $57.2 \pm 1.6$ & $63.4 \pm 1.8$ & $\mathbf{70.7 \pm 1.3}$ \\
\hline
Coil20 &  $56.9 \pm 1.8$   & $58.3 \pm 1.5$ & $60.4 \pm 1.3$ & $67.6 \pm 1.4$  & $62.2 \pm 1.7$  & $66.3 \pm 1.6$ & $68.6 \pm 1.7$ & $\mathbf{73.6 \pm 1.4 }$\\
\hline
\end{tabular}
\label{table_result_3accSVM}
\end{table*}

\begin{table*}[!ht]
\renewcommand{\arraystretch}{1.3}
\caption{\small Performance Comparison (Classification Accuracy $\pm$ std\%) of LDA, 2DPCA, 2DLDA, S2DLDA, P2DLDA, T-LPP, Bilinear SVM and CRP with five training data per class for each dataset. SVM is used as a classifier.}
\newcommand{\tabincell}[2]{\begin{tabular}{@{}#1@{}}#2\end{tabular}}
\centering
\begin{tabular}{|c||c|c|c|c|c|c|c|c|}
\hline
\bfseries Dataset & \bfseries LDA & \bfseries 2DPCA \cite{bib_Jian} & \bfseries 2DLDA \cite{bib_Ye}& \bfseries S2DLDA \cite{bib_Inoue}& \bfseries P2DLDA \cite{bib_Inoue}& \bfseries T-LPP \cite{bib_He_Ten}& \bfseries B-SVM \cite{bib_Pirsiavash}& \bfseries CRP \\
\hline \hline
UUIm &  $28.2 \pm 1.4$    & $27.7 \pm 1.2$ & $30.2 \pm 1.6$ & $28.4 \pm 1.9$  & $30.1 \pm 1.5$  & $31.8 \pm 1.6$ & $31.2 \pm 1.8$ & $\mathbf{33.5 \pm 1.2}$\\
\hline
CVL &   $37.6 \pm 1.4$   & $36.4 \pm 1.7$ & $40.8 \pm 1.9$ & $40.9 \pm 1.5$  & $31.7 \pm 1.8$  & $47.3 \pm 1.6$ & $49.3 \pm 1.3$ & $\mathbf{58.2 \pm 1.5}$\\
\hline
Pointing04(tilt) &  $13.2 \pm 1.8$    & $13.9 \pm 1.4$ & $14.6 \pm 1.7$ & $15.4 \pm 1.6$  & $14.7 \pm 1.4$  & $16.4 \pm 1.8$ & $19.8 \pm 1.9$ & $\mathbf{24.1 \pm 1.3}$\\
\hline
Pointing04(pan) & $8.8 \pm 1.8$  & $10.8 \pm 1.5$ & $10.3 \pm 1.3$ & $9.2 \pm 1.4$  & $8.8 \pm 1.2$  & $10.9 \pm 1.7$ & $12.4 \pm 1.8$ & $\mathbf{15.8 \pm 1.2}$\\
\hline
USPS &     $64.2 \pm 1.7$    & $72.9 \pm 1.4$ & $66.8 \pm 1.6$ & $67.8 \pm 1.2$  & $56.8 \pm 1.5$  & $65.3 \pm 1.2$ & $72.9 \pm 1.4$ & $\mathbf{75.8 \pm 1.3}$\\
\hline
COIL-20 &   $66.2 \pm 1.2$  & $68.1 \pm 1.9$ & $71.6 \pm 1.8$ & $80.2 \pm 1.5$  & $67.3 \pm 1.4$  & $76.3 \pm 1.1$ & $78.5 \pm 1.6$ & $\mathbf{82.4 \pm 1.4}$\\
\hline
\end{tabular}
\label{table_result_5accSVM}
\end{table*}

\begin{table*}[!ht]
\renewcommand{\arraystretch}{1.3}
\caption{\small Performance Comparison (Classification Accuracy $\pm$ std\%) of LDA, 2DPCA, 2DLDA, S2DLDA, P2DLDA, T-LPP, Bilinear SVM and CRP with ten training data per class for each dataset. SVM is used as a classifier.}
\newcommand{\tabincell}[2]{\begin{tabular}{@{}#1@{}}#2\end{tabular}}
\centering
\begin{tabular}{|c||c|c|c|c|c|c|c|c|}
\hline
\bfseries Dataset & \bfseries LDA & \bfseries 2DPCA \cite{bib_Jian} & \bfseries 2DLDA \cite{bib_Ye}& \bfseries S2DLDA \cite{bib_Inoue}& \bfseries P2DLDA \cite{bib_Inoue}& \bfseries T-LPP \cite{bib_He_Ten}& \bfseries B-SVM \cite{bib_Pirsiavash}& \bfseries CRP \\
\hline \hline
UUIm &   $36.1 \pm 1.6$  & $37.2 \pm 1.4$ & $39.5 \pm 1.7$ & $34.1 \pm 1.8$  & $41.2 \pm 1.2$  & $43.9 \pm 1.6$ & $48.2 \pm 1.4$ & $\mathbf{52.3 \pm 1.3}$\\
\hline
CVL &   $56.8 \pm 1.5$   & $57.1 \pm 1.6$ & $57.9 \pm 1.9$ & $51.9 \pm 1.4$  & $48.4 \pm 1.3$  & $66.9 \pm 1.6$ & $64.2 \pm 1.9$ & $\mathbf{68.3 \pm 1.5}$\\
\hline
Pointing04(tilt) & $23.9 \pm 1.7$   & $25.6 \pm 1.9$ & $25.8 \pm 1.4$ & $21.8 \pm 1.2$  & $21.8 \pm 1.2$  & $26.1 \pm 1.4$ & $28.8 \pm 1.3$ & $\mathbf{33.1 \pm 1.4}$\\
\hline
Pointing04(pan) &   $10.1 \pm 1.3$ & $10.5 \pm 1.6$ & $11.4 \pm 1.8$ & $10.5 \pm 1.6$  & $9.8 \pm 1.7$  & $13.6 \pm 1.5$ & $15.9 \pm 1.8$ & $\mathbf{19.4 \pm 1.6}$\\
\hline
USPS &  $73.4 \pm 1.7$   &$74.4 \pm 1.9$ & $77.3 \pm 1.3$ & $79.2 \pm 1.8$  & $73.9 \pm 1.4$  & $78.2 \pm 1.5$ & $79.4 \pm 1.9$ & $\mathbf{84.3 \pm 1.3}$\\
\hline
COIL-20 & $90.2 \pm 1.6$  & $93.6 \pm 1.4$ & $91.9 \pm 1.3$ & $91.7 \pm 1.8$  & $79.4 \pm 1.4$  & $90.8 \pm 1.8$ & $92.8 \pm 1.5$ & $\mathbf{98.8 \pm 1.1}$\\
\hline
\end{tabular}
\label{table_result_10accSVM}
\end{table*}

\begin{table*}[!ht]
\renewcommand{\arraystretch}{1.3}
\caption{\small Performance Comparison (Classification Accuracy $\pm$ std\%) of LDA, 2DPCA, 2DLDA, S2DLDA, P2DLDA, T-LPP, Bilinear SVM and CRP with 20 training data per class for each dataset. SVM is used as a classifier.}
\newcommand{\tabincell}[2]{\begin{tabular}{@{}#1@{}}#2\end{tabular}}
\centering
\begin{tabular}{|c||c|c|c|c|c|c|c|c|}
\hline
\bfseries Dataset & \bfseries LDA & \bfseries 2DPCA \cite{bib_Jian} & \bfseries 2DLDA \cite{bib_Ye}& \bfseries S2DLDA \cite{bib_Inoue}& \bfseries P2DLDA \cite{bib_Inoue}& \bfseries T-LPP \cite{bib_He_Ten}& \bfseries B-SVM \cite{bib_Pirsiavash}& \bfseries CRP \\
\hline \hline
UUIm &  $46.4 \pm 1.7$  & $46.1 \pm 1.4$ & $47.3 \pm 1.8$ & $45.8 \pm 1.4$  & $59.5 \pm 1.7$  & $56.3 \pm 1.3$ & $61.2 \pm 1.8$ & $\mathbf{64.2 \pm 1.3}$\\
\hline
CVL & $67.9 \pm 1.7$  & $68.3 \pm 1.4$ & $69.2 \pm 1.6$ & $68.6 \pm 1.7$  & $58.3 \pm 1.7$  & $69.1 \pm 1.4$ & $70.9 \pm 1.6$ & $\mathbf{79.2 \pm 1.5}$\\
\hline
Pointing04(tilt) &   $28.6 \pm 1.6$  & $29.1 \pm 1.5$ & $31.8 \pm 1.8$ & $27.9 \pm 1.3$  & $25.7 \pm 1.8$  & $35.1 \pm 1.3$ & $35.6 \pm 1.8 $ & $\mathbf{38.3 \pm 1.6}$\\
\hline
Pointing04(pan) &   $12.0 \pm 1.6$   & $11.8 \pm 1.5$ & $12.5 \pm 1.7$ & $13.8 \pm 1.2$  & $11.2 \pm 1.3$  & $14.4 \pm 1.4$ & $16.9 \pm 1.7$ & $\mathbf{20.3 \pm 1.4}$\\
\hline
USPS &  $83.8 \pm 1.8$   & $84.2 \pm 1.9$ & $85.8 \pm 1.5$ & $86.8 \pm 1.9$  & $80.8 \pm 1.3$  & $81.8 \pm 1.6$ & $86.6 \pm 1.8$ & $\mathbf{88.4 \pm 1.4}$\\
\hline
COIL-20 &    $93.5 \pm 1.8$    & $94.8 \pm 1.7$ & $95.9 \pm 1.4$ & $97.5 \pm 1.8$  & $89.5 \pm 1.3$  & $96.6 \pm 1.8$ & $96.7 \pm 1.5$ & $\mathbf{99.2 \pm 1.3}$\\
\hline
\end{tabular}
\label{table_result_20accSVM}
\end{table*}

\begin{table*}[!ht]
\renewcommand{\arraystretch}{1.3}
\caption{\small Performance Comparison (Classification Accuracy $\pm$ std\%) of LDA, 2DPCA, 2DLDA, S2DLDA, P2DLDA, T-LPP, Bilinear SVM and CRP with 80\% training data for each dataset. SVM is used as a classifier.}
\newcommand{\tabincell}[2]{\begin{tabular}{@{}#1@{}}#2\end{tabular}}
\centering
\begin{tabular}{|c||c|c|c|c|c|c|c|c|}
\hline
\bfseries Dataset & \bfseries LDA & \bfseries 2DPCA \cite{bib_Jian} & \bfseries 2DLDA \cite{bib_Ye}& \bfseries S2DLDA \cite{bib_Inoue}& \bfseries P2DLDA \cite{bib_Inoue}& \bfseries T-LPP \cite{bib_He_Ten}& \bfseries B-SVM \cite{bib_Pirsiavash}& \bfseries CRP \\
\hline \hline
UUIm &  $88.2 \pm 1.4$  & $88.4 \pm 1.7$ & $88.6 \pm 1.6$ & $89.1 \pm 1.5$  & $89.4 \pm 1.7$  & $89.7 \pm 1.5$ & $90.2 \pm 1.4$ & $\mathbf{90.6 \pm 1.6}$\\
\hline
CVL & $87.3 \pm 1.5$  & $87.7 \pm 1.7$ & $88.2 \pm 1.4$ & $89.3 \pm 1.5$  & $89.9 \pm 1.1$  & $90.3 \pm 1.5$ & $90.7 \pm 1.4$ & $\mathbf{91.5 \pm 1.3}$\\
\hline
Pointing04(tilt) &   $50.4 \pm 1.2$  & $51.1 \pm 1.3$ & $51.3 \pm 1.4$ & $51.6 \pm 1.6$  & $52.1 \pm 1.4$  & $52.5 \pm 1.3$ & $53.3 \pm 1.5 $ & $\mathbf{53.9 \pm 1.3}$\\
\hline
Pointing04(pan) &   $37.9 \pm 1.7$   & $38.3 \pm 1.3$ & $38.6 \pm 1.5$ & $39.1 \pm 1.4$  & $38.5 \pm 1.1$  & $39.8 \pm 1.6$ & $40.5 \pm 1.4$ & $\mathbf{42.6 \pm 1.2}$\\
\hline
USPS &  $93.1 \pm 1.8$   & $93.5 \pm 1.5$ & $93.8 \pm 1.1$ & $94.1 \pm 1.2$  & $94.3 \pm 1.8$  & $94.7 \pm 1.4$ & $95.1 \pm 1.1$ & $\mathbf{95.6 \pm 1.3}$\\
\hline
COIL-20 &    $96.5 \pm 1.4$    & $95.9 \pm 1.5$ & $96.4 \pm 1.7$ & $97.5 \pm 1.8$  & $89.9 \pm 1.3$  & $98.3 \pm 1.5$ & $99.1 \pm 1.3$ & $\mathbf{99.4 \pm 0.9}$\\
\hline
\end{tabular}
\label{table_result_80percaccSVM}
\end{table*}

\begin{table*}[!ht]
\renewcommand{\arraystretch}{1.3}
\caption{\small Performance Comparison (Classification Accuracy $\pm$ std\%) of LDA, 2DPCA, 2DLDA, S2DLDA, P2DLDA, T-LPP, Bilinear SVM and CRP with three training data per class for each dataset. $1$-NN is used as a classifier.}
\newcommand{\tabincell}[2]{\begin{tabular}{@{}#1@{}}#2\end{tabular}}
\centering
\begin{tabular}{|c||c|c|c|c|c|c|c|c|}
\hline
\bfseries Dataset & \bfseries LDA & \bfseries 2DPCA \cite{bib_Jian} & \bfseries 2DLDA \cite{bib_Ye}& \bfseries S2DLDA \cite{bib_Inoue}& \bfseries P2DLDA \cite{bib_Inoue}& \bfseries T-LPP \cite{bib_He_Ten}& \bfseries B-SVM \cite{bib_Pirsiavash}& \bfseries CRP \\
\hline \hline
UUIm &   $24.3 \pm 1.8$  & $25.8 \pm 1.6$  & $27.4 \pm 1.9$ & $28.8 \pm 1.5$ & $24.2 \pm 1.2$ & $24.9 \pm 1.8$ & $28.2 \pm 1.7$ & $\mathbf{32.1 \pm 1.5}$\\
\hline
CVL &  $21.2 \pm 1.2$ & $23.4 \pm 1.9$  & $24.7 \pm 1.4$ & $29.2 \pm 1.8$ & $19.8 \pm 1.4$ & $35.2 \pm 1.1$ & $47.1 \pm 1.6$ & $\mathbf{52.1 \pm 1.3}$\\
\hline
Pointing'04(Tilt) &   $9.2 \pm 1.3$ & $9.5 \pm 1.0$ & $8.1 \pm 1.3$ & $12.9 \pm 0.8$ & $10.4 \pm 0.9$ & $15.8 \pm 1.4$ & $18.1 \pm 1.7$ & $\mathbf{23.6 \pm 0.9}$ \\
\hline
Pointing'04(Pan) &   $8.8 \pm 1.1$ & $9.7 \pm 0.9$ & $10.2 \pm 1.5$ & $9.9 \pm 1.2$ & $8.8 \pm 1.3$ & $11.1 \pm 1.4$ & $13.1 \pm 1.5$ & $\mathbf{15.1 \pm 1.1}$ \\
\hline
USPS &    $44.0 \pm 1.8$ & $44.3 \pm 1.4$ & $45.9 \pm 1.2$ & $54.2 \pm 1.6$ & $46.5 \pm 1.2$ & $52.8 \pm 1.5$ & $63.4 \pm 1.8$ & $\mathbf{72.3 \pm 1.4}$ \\
\hline
Coil20 &  $57.1 \pm 1.4$   & $58.2 \pm 1.8$ & $61.2 \pm 1.1$ & $67.6 \pm 1.8$  & $52.4 \pm 1.5$  & $60.3 \pm 1.7$ & $68.6 \pm 1.7$ & $\mathbf{74.9 \pm 1.1}$\\
\hline
\end{tabular}
\label{table_result_3accNN}
\end{table*}

\begin{table*}[!ht]
\renewcommand{\arraystretch}{1.3}
\caption{\small Performance Comparison (Classification Accuracy $\pm$ std\%) of LDA, 2DPCA, 2DLDA, S2DLDA, P2DLDA, T-LPP, Bilinear SVM and CRP with five training data per class for each dataset. 1NN is used as a classifier.}
\newcommand{\tabincell}[2]{\begin{tabular}{@{}#1@{}}#2\end{tabular}}
\centering
\begin{tabular}{|c||c|c|c|c|c|c|c|c|}
\hline
\bfseries Dataset & \bfseries LDA & \bfseries 2DPCA \cite{bib_Jian} & \bfseries 2DLDA \cite{bib_Ye}& \bfseries S2DLDA \cite{bib_Inoue}& \bfseries P2DLDA \cite{bib_Inoue}& \bfseries T-LPP \cite{bib_He_Ten}& \bfseries B-SVM \cite{bib_Pirsiavash}& \bfseries CRP \\
\hline \hline
UUIm &   $30.2 \pm 1.6$  & $31.6 \pm 1.3$ & $34.5 \pm 1.7$ & $\mathbf{35.4 \pm 1.5}$  & $33.2 \pm 1.2$  & $31.0 \pm 1.7$ & $31.2 \pm 1.8$ & $34.8 \pm 1.3$\\
\hline
CVL &   $32.8 \pm 1.3$  & $33.3 \pm 1.6$ & $35.1 \pm 1.8$ & $36.6 \pm 1.4$  & $30.8 \pm 1.2$  & $44.2 \pm 1.6$ & $49.3 \pm 1.3$ & $\mathbf{52.9 \pm 1.5}$\\
\hline
Pointing04(tilt) &   $12.1 \pm 1.8$   & $12.8 \pm 1.4$ & $14.4 \pm 1.6$ & $15.2 \pm 1.7$  & $10.2 \pm 1.4$  & $19.8 \pm 1.9$ & $19.8 \pm 1.9$ & $\mathbf{27.9 \pm 1.1}$\\
\hline
Pointing04(pan) &   $9.2 \pm 1.3$   & $14.2 \pm 1.8$ & $9.8 \pm 1.5$ & $8.8 \pm 1.7$  & $7.9 \pm 1.4$  & $10.4 \pm 1.7$ & $12.4 \pm 1.8$ & $\mathbf{16.2 \pm 1.4}$\\
\hline
USPS &     $61.2 \pm 1.8$     & $63.8 \pm 1.4$ & $64.2 \pm 1.9$ & $66.1 \pm 1.7$  & $45.3 \pm 1.3$  & $68.5 \pm 1.1$ & $72.9 \pm 1.4$ & $\mathbf{78.1 \pm 1.6}$\\
\hline
COIL-20 &   $69.6 \pm 2.2$  & $71.3 \pm 1.8$ & $72.9 \pm 1.4$ & $76.2 \pm 1.9$  & $59.2 \pm 1.3$  & $76.5 \pm 2.1$ & $78.5 \pm 1.6$ & $\mathbf{83.2 \pm 1.3}$\\
\hline
\end{tabular}
\label{table_result_5accNN}
\end{table*}

\begin{table*}[!ht]
\renewcommand{\arraystretch}{1.3}
\caption{\small Performance Comparison (Classification Accuracy $\pm$ std\%) of LDA, 2DPCA, 2DLDA, S2DLDA, P2DLDA, T-LPP, Bilinear SVM and CRP with ten training data per class for each dataset. 1NN is used as a classifier.}
\newcommand{\tabincell}[2]{\begin{tabular}{@{}#1@{}}#2\end{tabular}}
\centering
\begin{tabular}{|c||c|c|c|c|c|c|c|c|}
\hline
\bfseries Dataset & \bfseries LDA & \bfseries 2DPCA \cite{bib_Jian} & \bfseries 2DLDA \cite{bib_Ye}& \bfseries S2DLDA \cite{bib_Inoue}& \bfseries P2DLDA \cite{bib_Inoue}& \bfseries T-LPP \cite{bib_He_Ten}& \bfseries B-SVM \cite{bib_Pirsiavash}& \bfseries CRP \\
\hline \hline
UUIm &    $43.8 \pm 1.7$   & $44.2 \pm 1.3$ & $46.1 \pm 1.8$ & $47.3 \pm 1.6$  & $40.3 \pm 1.8$  & $44.9 \pm 1.4$ & $48.2 \pm 1.6$ & $\mathbf{53.2 \pm 1.4}$\\
\hline
CVL &      $47.3 \pm 1.4$  & $47.9 \pm 1.5$ & $50.3 \pm 1.8$ & $51.9 \pm 1.5$  & $46.9 \pm 1.4$  & $55.1 \pm 1.6$ & $64.2 \pm 1.9$ & $\mathbf{67.3 \pm 1.3}$\\
\hline
Pointing04(tilt) &  $22.9 \pm 1.7$   & $22.7 \pm 1.5$ & $24.8 \pm 1.3$ & $25.3 \pm 1.8$  & $14.2 \pm 1.5$  & $20.5 \pm 1.2$ & $28.8 \pm 1.3$ & $\mathbf{32.4 \pm 1.3}$\\
\hline
Pointing04(pan) &  $9.6 \pm 1.5$ & $9.3 \pm 1.2$ & $10.7 \pm 1.4$ & $10.9 \pm 1.8$  & $8.8 \pm 1.4$  & $11.3 \pm 1.7$ & $15.9 \pm 1.8$ & $\mathbf{19.2 \pm 1.5}$\\
\hline
USPS &    $68.5 \pm 1.7$   & $69.2 \pm 1.5$ & $71.8 \pm 1.4$ & $71.2 \pm 1.6$  & $64.8 \pm 1.9$  & $76.7 \pm 1.6$ & $79.4 \pm 1.9$ & $\mathbf{84.4 \pm 1.3}$\\
\hline
COIL-20 &  $85.6 \pm 1.6$ & $86.1 \pm 1.8$ & $88.8 \pm 1.4$ & $92.1 \pm 1.7$  & $80.5 \pm 1.6$  & $89.4 \pm 1.8$ & $92.8 \pm 1.5$ & $\mathbf{95.1 \pm 1.5}$\\
\hline
\end{tabular}
\label{table_result_10accNN}
\end{table*}

\begin{table*}[!ht]
\renewcommand{\arraystretch}{1.3}
\caption{\small Performance Comparison (Classification Accuracy $\pm$ std\%) of LDA, 2DPCA, 2DLDA, S2DLDA, P2DLDA, T-LPP, Bilinear SVM and CRP with 20 training data per class for each dataset. 1NN is used as a classifier.}
\newcommand{\tabincell}[2]{\begin{tabular}{@{}#1@{}}#2\end{tabular}}
\centering
\begin{tabular}{|c||c|c|c|c|c|c|c|c|}
\hline
\bfseries Dataset & \bfseries LDA & \bfseries 2DPCA \cite{bib_Jian} & \bfseries 2DLDA \cite{bib_Ye}& \bfseries S2DLDA \cite{bib_Inoue}& \bfseries P2DLDA \cite{bib_Inoue}& \bfseries T-LPP \cite{bib_He_Ten}& \bfseries B-SVM \cite{bib_Pirsiavash}& \bfseries CRP \\
\hline \hline
UUIm &     $58.0 \pm 1.3$   & $58.2 \pm 1.5$ & $61.4 \pm 1.8$ & $61.9 \pm 1.2$  & $58.3 \pm 1.8$  & $60.9 \pm 1.3$ & $61.2 \pm 1.8$ & $\mathbf{68.4 \pm 1.3}$\\
\hline
CVL &    $63.7 \pm 1.3$   & $64.1 \pm 1.8$ & $66.7 \pm 1.2$ & $67.3 \pm 1.3$  & $60.2 \pm 1.7$  & $65.2 \pm 1.9$ & $70.9 \pm 1.6$ & $\mathbf{74.2 \pm 1.6}$\\
\hline
Pointing04(tilt) &   $32.7 \pm 1.8$   & $30.2 \pm 1.3$ & $37.8 \pm 1.5$ & $38.6 \pm 1.3$  & $26.7 \pm 1.7$  & $33.6 \pm 1.9$ & $35.6 \pm 1.8$ & $\mathbf{42.3 \pm 1.4}$\\
\hline
Pointing04(pan) & $11.4 \pm 1.8$  & $11.8 \pm 1.4$ & $14.7 \pm 1.7$ & $15.2 \pm 1.4$  & $10.0 \pm 1.6$  & $13.4 \pm 1.8$ & $16.9 \pm 1.7$ & $\mathbf{22.3 \pm 1.8}$\\
\hline
USPS &   $83.1 \pm 1.3$    & $83.6 \pm 1.4$ & $84.5 \pm 1.8$ & $85.6 \pm 1.2$  & $74.7 \pm 1.5$  & $79.5 \pm 1.4$ & $86.6 \pm 1.8$ & $\mathbf{89.2 \pm 1.4}$\\
\hline
COIL-20 &   $94.0 \pm 1.4$   & $94.8 \pm 1.8$ & $95.9 \pm 1.6$ & $95.6 \pm 1.2$  & $77.9 \pm 1.8$  & $92.1 \pm 1.4$ & $96.7 \pm 1.5$ & $\mathbf{98.8 \pm 1.2}$\\
\hline
\end{tabular}
\label{table_result_20accNN}
\end{table*}

\begin{table*}[!ht]
\renewcommand{\arraystretch}{1.3}
\caption{\small Performance Comparison (Classification Accuracy $\pm$ std\%) of LDA, 2DPCA, 2DLDA, S2DLDA, P2DLDA, T-LPP, Bilinear SVM and CRP with 80\% training data for each dataset. 1NN is used as a classifier.}
\newcommand{\tabincell}[2]{\begin{tabular}{@{}#1@{}}#2\end{tabular}}
\centering
\begin{tabular}{|c||c|c|c|c|c|c|c|c|}
\hline
\bfseries Dataset & \bfseries LDA & \bfseries 2DPCA \cite{bib_Jian} & \bfseries 2DLDA \cite{bib_Ye}& \bfseries S2DLDA \cite{bib_Inoue}& \bfseries P2DLDA \cite{bib_Inoue}& \bfseries T-LPP \cite{bib_He_Ten}& \bfseries B-SVM \cite{bib_Pirsiavash}& \bfseries CRP \\
\hline \hline
UUIm &     $89.7 \pm 1.5$   & $90.1 \pm 1.8$ & $90.6 \pm 1.5$ & $91.2 \pm 1.4$  & $91.8 \pm 1.2$  & $92.3 \pm 1.4$ & $93.0 \pm 1.6$ & $\mathbf{93.5 \pm 1.5}$\\
\hline
CVL &    $90.2 \pm 1.6$   & $90.6 \pm 1.5$ & $91.1 \pm 1.4$ & $92.4 \pm 1.5$  & $92.8 \pm 1.3$  & $93.1 \pm 1.3$ & $93.4 \pm 1.5$ & $\mathbf{94.7 \pm 1.3}$\\
\hline
Pointing04(tilt) &   $47.5 \pm 1.6$   & $48.1 \pm 1.6$ & $48.5 \pm 1.2$ & $48.8 \pm 1.7$  & $50.2 \pm 1.5$  & $51.3 \pm 1.4$ & $52.1 \pm 1.7$ & $\mathbf{56.4 \pm 1.6}$\\
\hline
Pointing04(pan) & $36.9 \pm 1.5$  & $37.4 \pm 1.6$ & $37.8 \pm 1.4$ & $38.1 \pm 1.6$  & $38.5 \pm 1.4$  & $40.8 \pm 1.2$ & $41.4 \pm 1.3$ & $\mathbf{45.1 \pm 1.5}$\\
\hline
USPS &   $94.8 \pm 1.3$    & $95.2 \pm 1.1$ & $95.6 \pm 1.2$ & $96.4 \pm 1.4$  & $94.5 \pm 1.2$  & $95.1 \pm 1.1$ & $96.2 \pm 1.4$ & $\mathbf{96.8 \pm 0.9}$\\
\hline
COIL-20 &   $96.1 \pm 1.1$   & $96.4 \pm 0.8$ & $96.7 \pm 1.1$ & $97.1 \pm 0.7$  & $89.3 \pm 0.9$  & $97.5 \pm 0.8$ & $98.9 \pm 0.6$ & $\mathbf{99.6 \pm 0.3}$\\
\hline
\end{tabular}
\label{table_result_80percaccNN}
\end{table*}

\subsubsection{Performance Comparison in Two-class Setting}
In this section, we compare the proposed algorithm with the other related algorithms in a two-class setting. Three face datasets, including UMIST, yaleB, and ORL, are utilized to evaluate the performance of gender recognition, which is a binary class problem. Note that there are only two classes in this case. We project the original data into a $5^2$ dimensional subspace to exploit enough discriminant information for all the compared algorithms empirically. In this experiment, we randomly sample 3 data per class as the training data. The experimental results are shown in Tab. \ref{table_result_linear}. From the experimental results, we can observe that the proposed CRP can perform better than all the compared algorithms in the two-class setting.

\begin{table*}[!ht]
\renewcommand{\arraystretch}{1.3}
\caption{Performance Comparison w.r.t two-class setting (Classification Accuracy$\pm$ std\%) of LDA, 2DPCA, 2DLDA, S2DLDA, P2DLDA, T-LPP, Bilinear SVM and CRP with 80\% training data for each dataset.}
\newcommand{\tabincell}[2]{\begin{tabular}{@{}#1@{}}#2\end{tabular}}
\centering
\begin{tabular}{|c||c|c|c|c|c|c|c|c|}
\hline
\bfseries Dataset & \bfseries LDA & \bfseries 2DPCA \cite{bib_Jian} & \bfseries 2DLDA \cite{bib_Ye}& \bfseries S2DLDA \cite{bib_Inoue}& \bfseries P2DLDA \cite{bib_Inoue}& \bfseries T-LPP \cite{bib_He_Ten}& \bfseries B-SVM \cite{bib_Pirsiavash}& \bfseries CRP \\
\hline \hline
ORL &  $89.8 \pm 1.5$   & $90.5 \pm 1.2$  & $91.1 \pm 1.4$ & $91.7 \pm 1.3$ & $92.3 \pm 1.1$ & $92.9 \pm 1.4$ & $93.3 \pm 1.3$ & $\mathbf{93.9 \pm 1.2}$\\
\hline
UMIST &    $83.1 \pm 1.4$  & $84.2 \pm 1.6$  & $84.7 \pm 1.5$ & $85.7 \pm 1.4$ & $86.2 \pm 1.8$ & $86.5 \pm 1.3$ & $87.2 \pm 1.6$ & $\mathbf{89.4 \pm 1.1}$\\
\hline
yaleB &  $74.6 \pm 1.4$  & $75.1 \pm 1.6$ & $75.6 \pm 1.7$ & $76.2 \pm 1.3$ & $76.5 \pm 1.6$ & $76.9 \pm 1.3$ & $77.4 \pm 1.8$ & $\mathbf{78.6 \pm 1.4}$ \\
\hline
\end{tabular}
\label{table_result_linear}
\end{table*}

\section{Conclusion}
It is more natural to represent the real world application data as matrices since we can preserve the spatial correlations while avoiding the curse of dimensionality. In this paper, we propose a novel compound rank-k projection algorithm for bilinear analysis. Our approach directly deal with the matrices. In this way, the spatial correlations can be preserved, and computation complexity can be decreased.  Our approach achieves better performance than the classical two-dimensional linear discriminant analysis because our method exploits the multiple projection models and promises a monotonic increase of the objective function value. Hence, the optimum can be obtained. The major novelty of our approach is that multiple projection models are used to provide a larger space in which the local optimal is obtained. Consequently, CRP can get better performances. 

In machine learning and statistics, ensemble methods use multiple learning algorithms to obtain better predictive performance than could be obtained from any of the constituent learning algorithms. In some sense, the proposed algorithm has a close relationship to ensemble learning. Our algorithm differs from ensemble learning in the following two aspects: First, in a typical ensemble learning scenario, the multiple learning algorithms usually adopt different criteria, e.g., to combine the SVM classifier and the least regression classifier for classification. Differently, the proposed algorithm is a single algorithm (as opposed multiple algorithms) designed in a single framework, i.e., the Linear Discriminant Analysis framework. Second, in ensemble learning, the output of a single constituent learning algorithm can be directly regarded as the final results. For example, the output of an SVM classifier can be used as the classification result, even though the performance could be improved if we combine SVM with least square regression. In contrast, the output of each projection model of our algorithm corresponds to only one dimension of the subspace, and thus we cannot use the output of a single model to represent the original data. We must use all the  models to obtain a -dimensional representation.

Our experimental results validate that our proposed algorithm outperform the other compared algorithms in terms of classification accuracy using two different classifiers.

\appendix[Lemmas used in this paper]

\begin{lemma}\label{lemma1}
\begin{equation}\nonumber
vec(UV^T)(vec(UV^T))^Tvec(X) = Tr(U^TXV)vec(UV^T),
\end{equation} where $vec(\cdot)$ represents the vectorization of a matrix. 
\end{lemma}

\begin{proof}
By substituting Lemma \ref{lemma2}, we can obtain

\begin{equation}\nonumber
\begin{split}
vec(UV^T)(vec(UV^T))^T & =\sum\limits_i v_i \otimes u_i (\sum\limits_i v_i \otimes u_i)^T\\
& =\sum\limits_{i,j} (v_i \otimes u_i)(v_j \otimes u_j)^T
\end{split}
\end{equation}

According to the properties of Kronecker product, we can obtain

\begin{equation}\nonumber
\begin{split}
\sum\limits_{i,j} (v_i \otimes u_i)(v_j \otimes u_j)^T & = \sum_{i,j} (v_iv_j^T \otimes u_iu_j^T)
\end{split}
\end{equation}

By multiplying $vec(X)$, we can get 

\begin{equation}\nonumber
\begin{split}
\sum_{i,j} (v_iv_j^T \otimes u_iu_j^T) vec(X) & = vec(\sum_j u_j^Txv_j \sum_i u_iv_i^T)\\
& = vec(\sum_j u_j^Txv_j \sum_i u_iv_i^T)\\
& = Tr(U^TXV)vec(UV^T)
\end{split}
\end{equation}

The result follows.
\end{proof}

\begin{lemma}\label{lemma2}
\begin{equation}\nonumber
vec(UV^T) = \sum_i v_i \otimes u_i,
\end{equation} where $\otimes$ is the Kronecker product.
\end{lemma}

\begin{proof}
Referring to Page 26, Chapter 2 of "Matrix Calculus and the Kronecker Product with Applications and C++ Programs" \cite{bib_Steeb}.
\end{proof}

\begin{lemma}\label{lemma3}
\begin{equation}\nonumber
vec(UV^T)^Tvec(UV^T)=Tr(U^TUV^TV)
\end{equation}
\end{lemma}

\begin{proof}
Substituting Lemma \ref{lemma2}, we can obtain
\begin{equation}\nonumber
\begin{split}
(vec(UV^T))^Tvec(UV^T)
& =(\sum\limits_i v_i \otimes u_i)^T \sum\limits_i (v_i \otimes u_i)\\
& =\sum\limits_{i,j} (v_j \otimes u_j)^T(v_i \otimes u_i)\\
\end{split}
\end{equation}

For any two vectors, $u$ and $v$, we have properties of Kronecker product as follows:

\begin{equation}\nonumber
\begin{split}
\sum\limits_{i,j} (v_j \otimes u_j)^T(v_i \otimes u_i) & = \sum_{i,j} (v_j^Tv_i \otimes  u_j^Tu_i)\\
& = \sum_{i,j} (v_j^Tv_iu_j^Tu_i) 
\end{split}
\end{equation}

According to properties of dot product, we have

\begin{equation}\nonumber
\begin{split}
\sum_{i,j} u_j^Tu_i = \sum_{i,j} u_i^Tu_j
\end{split}
\end{equation}

Hence we can obtain

\begin{equation}\nonumber
\begin{split}
\sum_{i,j} (v_j^Tv_iu_i^Tu_j) & = \sum_{i,j} (u_i^T(u_jv_j^T)v_i)\\
& = \sum_i (u_i^T \sum_j (u_jv_j^T)v_i)\\
& = Tr(U^TUV^TV)
\end{split}
\end{equation}
The result follows.
\end{proof}

\begin{lemma}\label{lemma5}
Suppose A, B, X are matrices, we have $vec(AXB)=(B^T \otimes A)vec(X)$.
\end {lemma}

\begin{proof}
As proved in \cite{bib_Steeb}, the one vector of order $n$ obeys the relation $e = \sum e_i$.

Similarly, 
\begin{equation}\nonumber
\begin{split}
(AXB)_k & = \sum_{j}(B_{jk}A)X_j\\
& = [B_{1k}A~B_{2k}A ... B_{nk}A] \left[\begin{array}{c}X_1\\X_2\\...\\X_n
\end{array}\right] \\
& = [B_k^T \otimes A]vec(X)\\
\end{split}
\end{equation}

The result follows.
\end{proof}

\begin{lemma}\label{lemma4}
\begin{equation}\nonumber
Tr(A^TBC)=vec(A)^T(I \otimes B)vec(C)
\end{equation}
\end{lemma}

\begin{proof}

By the trace definition, we have
\begin{equation}\nonumber
\begin{split}
Tr(A^TB) & = vec(A)^Tvec(B) 
\end{split}
\end{equation}

According to Lemma \ref{lemma5}, we can obtain

\begin{equation}\nonumber
\begin{split}
vec(BC) & = vec(BCI)\\
& = (I \otimes B)vec(C)
\end{split}
\end{equation}

By incorporating the above two equations, we can get
\begin{equation}\nonumber
\begin{split}
Tr(A^TBC) & =vec(A)^Tvec(BC)\\
& =vec(A)^T(I \otimes B)vec(C)
\end{split}
\end{equation}

The result follows.
\end{proof}

\begin{lemma}\label{lemma6}
\begin{equation}\nonumber
Tr(A^TBC)=vec(A)^T(C^T \otimes I)vec(B)
\end{equation}
\end{lemma}

\begin{proof}
Similarly, according to the trace definition, we have 
\begin{equation}\nonumber
Tr(A^TB)=vec(A)^Tvec(B),
\end{equation}
so 
\begin{equation}\nonumber
\begin{split}
Tr(A^TBC) & =vec(A)^Tvec(BC)\\
& =vec(A)^T(C^T \otimes I)vec(B)
\end{split}
\end{equation}.

The result follows.
\end{proof}

%

{
\bibliographystyle{IEEEtran}
\bibliography{dma}
}

\begin{IEEEbiography}[{\includegraphics[width=1in,height=1.25in,clip,keepaspectratio]{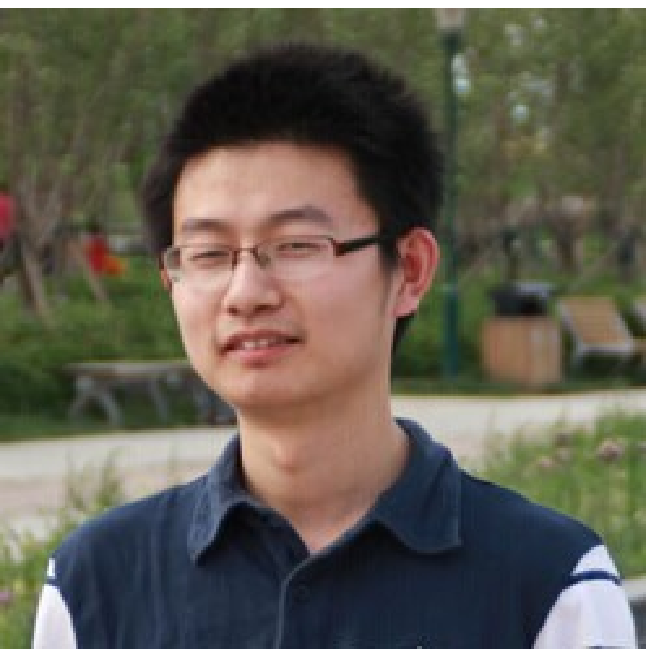}}]{Xiaojun Chang}
	is a Ph.D. student at University of Technology Sydney, under the supervision of Dr. Yi Yang. His research interests include machine learning, data mining and computer vision. His publications appear in proceedings of prestigious international conference like ICML, AAAI, IJCAI and etc. 
\end{IEEEbiography}

\begin{IEEEbiography}[{\includegraphics[width=1in,height=1.25in,clip,keepaspectratio]{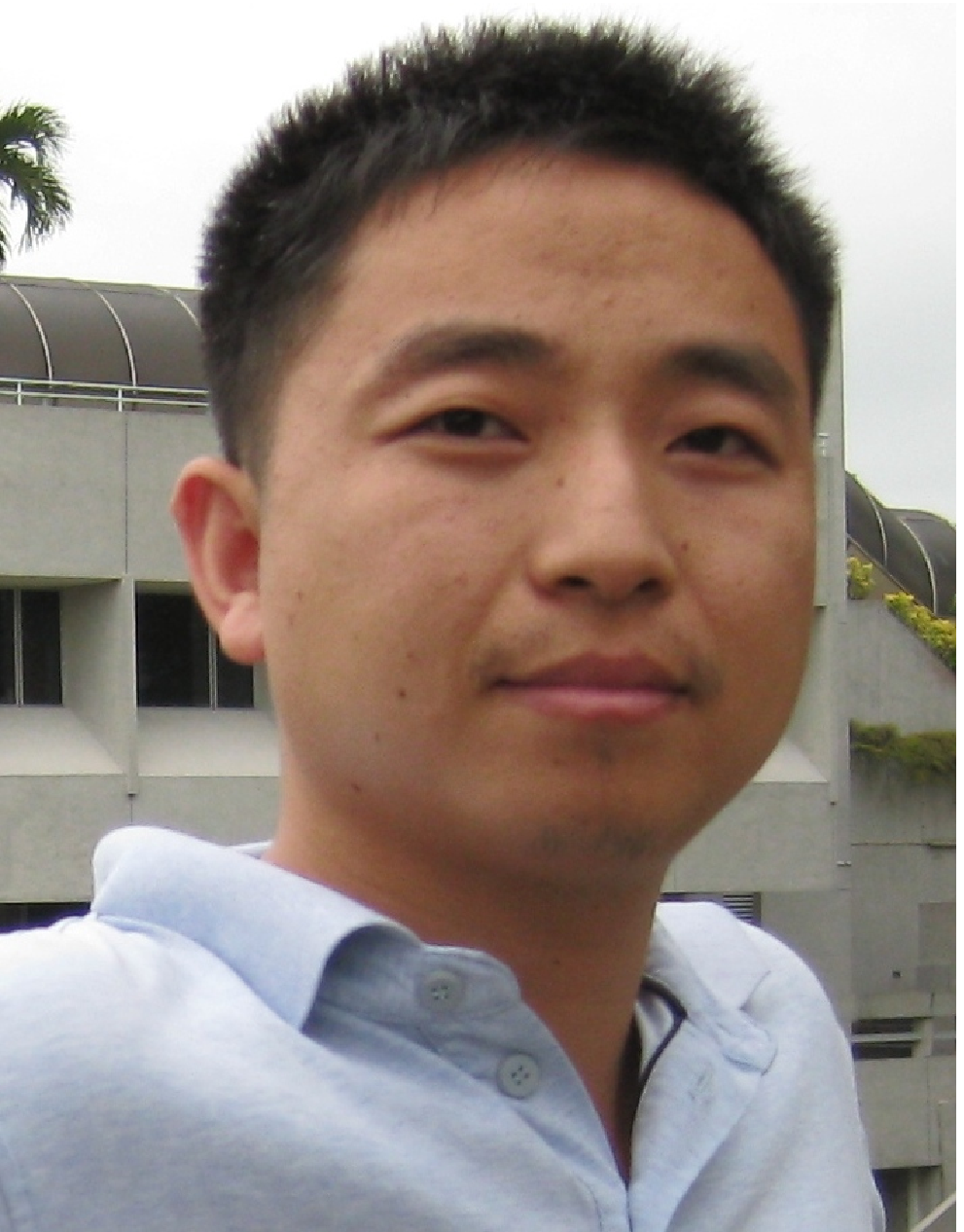}}]{Feiping Nie}
	received the Ph.D. degree in computer science from Tsinghua University, Beijing, China in 2009. He is currently a Professor with Center for OPTical Imagery Analysis and Learning, Northwestern Polytechnical University, Shaanxi, China. His research interests are machine learning and its applications fields, such as pattern recognition, data mining,computer vision, image processing and information retrieval. He has published more than 100 papers in the prestigious journals and conferences like TPAMI, TKDE, ICML, NIPS, KDD, and etc. He is now serving as Associate Editor or PC member for several prestigious journals and conferences in the related fields.
\end{IEEEbiography}

\begin{IEEEbiography}[{\includegraphics[width=1in,height=1.25in,clip,keepaspectratio]{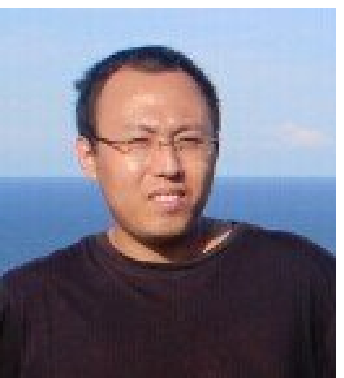}}]{Sen Wang}
	received his Ph.D. in School of Information Technology and Electrical Engineering, The University of Queensland, in 2014. He is currently a Australian Research Council Post-doctoral Research Fellow in DKE Group supervised by A. Prof. Xue Li. His research interest includes data mining, pattern recognition, and relevant applications in medical data and social media analysis.
\end{IEEEbiography}

\begin{IEEEbiography}[{\includegraphics[width=1in,height=1.25in,clip,keepaspectratio]{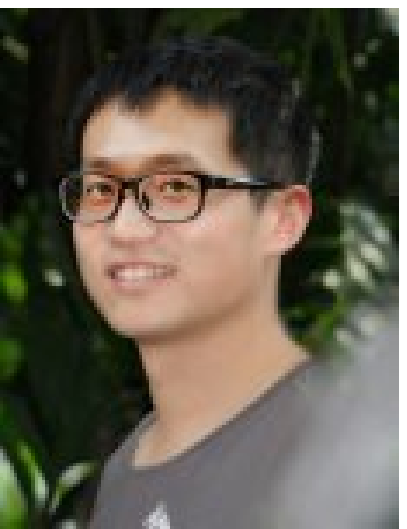}}]{Yi Yang}
	received the Ph.D. degree in computer science from Zhejiang University, Hangzhou, China, in 2010. He is currently a senior lecturer with University of Technology Sydney, Australia. He was a Post-Doctoral Research with the School of Computer Science, Carnegie Mellon University, Pittsburgh, PA, USA. His current research interest include machine learning and its applications to multimedia content analysis and computer vision, such as multimedia indexing and retrieval, surveillance video analysis and video semantics understanding.
\end{IEEEbiography}

\begin{IEEEbiography}[{\includegraphics[width=1in,height=1.25in,clip,keepaspectratio]{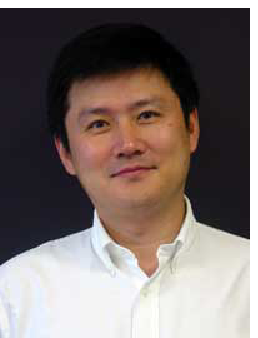}}]{Xiaofang Zhou}
	received the B.S. and M.S. degrees in computer science from Nanjing University, Nanjing, China, and the Ph.D. degree in computer science from The University of Queensland, Brisbane, QLD, Australia, in 1984, 1987 and 1994, respectively. He is a Professor of Computer Science with the University of Queensland. He is the Head of the Data and Knowledge Engineering Research Division, School of Information Technology and Electrical Engineering. His current research interests include spatial and multimedia databases, high performance query processing, web information systems, data mining, bioinformatics, and e-research.
\end{IEEEbiography}

\begin{IEEEbiography}[{\includegraphics[width=1in,height=1.25in,clip,keepaspectratio]{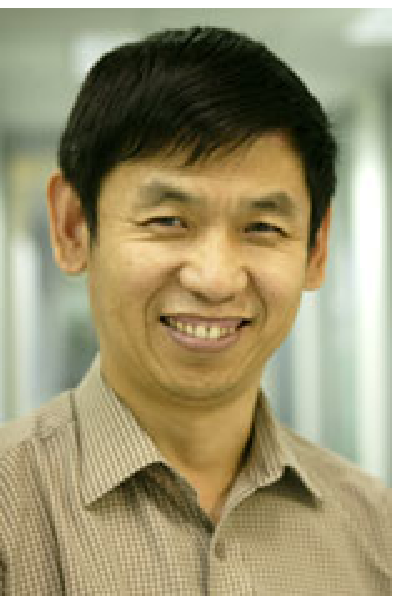}}]{Chengqi Zhang}
	received the Ph.D. degree in computer science from The University of Queensland, Brisbane, Australia, in 1991, and the DrSc degree from Deakin University, Geelong, Australia, in 2002. He is currently with the University of Technology, Sydney (UTS), Sydney, Australia, where he is a research professor of information technology and the director of the UTS Priority Investment Research Center for Quantum Computation and Intelligent Systems. He has published more than 200 refereed research papers. His main research interests include data mining and its applications. He is a fellow of the Australian Computer Society.
\end{IEEEbiography}

\end{document}